\documentclass{article}
\usepackage{fullpage,url,hyperref}

\usepackage{graphicx} 
\usepackage{amsthm,amsmath,amssymb,bm,amsfonts,thmtools}
\usepackage{xcolor}
\usepackage{soul}
\usepackage{booktabs}
\usepackage{tabularx}
\usepackage{natbib}
\usepackage{hyperref}

\usepackage{algorithm}
\usepackage{algpseudocode}

\usepackage{makecell}

\usepackage{tikz}
\usetikzlibrary{bayesnet}
\usetikzlibrary{calc,arrows.meta}

\newcommand\numberthis{\addtocounter{equation}{1}\tag{\theequation}}

\DeclareMathOperator*{\argmin}{\arg\!\min}

\newtheorem{theorem}{Theorem}

\newtheorem{lemma}[theorem]{Lemma}
\newtheorem{proposition}[theorem]{Proposition}

\theoremstyle{definition}
\newtheorem{definition}[theorem]{Definition}
\newtheorem{assumption}[theorem]{Assumption}
\newtheorem{example}[theorem]{Example}

\providecommand{\keywords}[1]
{
  \small	
  \noindent
  \textbf{Keywords:} #1
}
 
\title{Generalising maximum mean discrepancy:\\ kernelised functional Bregman divergences}
\author{Russell Tsuchida 
\thanks{Department of Data Science and AI, Monash University, Australia.} 
\and 
Frank Nielsen
\thanks{Sony Computer Science Laboratories, Inc. Tokyo, Japan.}}

 \date{}

\begin{document}

\maketitle

\begin{abstract}
Bregman divergences play a pivotal role in statistics, machine learning and computational information geometry. 
Particularly in the context of machine learning, they are central to clustering, exponential families, parameter estimation and optimisation, among other things. 
Despite this,  the full toolkit of Hilbert spaces and in particular reproducing kernel Hilbert spaces have not been systematically developed and applied to \emph{functional Bregman divergences}, where points are functions rather than  finite-dimensional parameter vectors.  
While other types of functional Bregman divergences have been studied, these are typically in a Banach space rather than more directly aligned with kernel methods and Hilbert-space geometry commonly used in machine learning.
We consider functional Bregman divergences on a Hilbert space, where the self-dual pairing and Riesz representer afford us particularly convenient calculus. 
Further specialising Bregman generators as a composition involving a kernel mean embedding makes such divergences easy to estimate. 
We discuss applications in clustering, universal estimation, robust estimation and generative modelling, and contrast our approach with other types of Bregman divergences.
\end{abstract}


\section{Introduction}
The classical Bregman divergence (BD) operating on finite-dimensional vectors is defined as the difference between an appropriately defined strictly convex function $\Phi$ evaluated at a point $a$ and the first order Taylor approximation of the function around a point $b$ evaluated at $a$~\citep{Bregman-1967,banerjee2005,nielsen2020elementary},
$$
d_{\Phi}(a,b) = \Phi(a) - \Phi(b) - \langle \nabla \Phi(b), a-b \rangle_{\mathbb{R}^d}.
$$
Such Bregman divergences offer helpful tools for constructing and analysing algorithms in machine learning and statistics. 
There is a unique BD associated to every sufficiently regular exponential family~\citep{banerjee2005,frongillo2014convex}, providing a link between statistics and geometry underlying the Bregman divergences. 
This link, as well as other nice properties, facilitates the use of BDs in a wide range of applications in statistics and machine learning, including density estimation~\citep{azoury2001relative}, boosting~\citep{collins2002logistic}, parameter estimation~\cite{barndorff2014information}, optimisation algorithms~\citep{ProximalGenBD-1997,beck2003mirror,raskutti2015information}, dimensionality reduction~\citep{collins2001generalization}, clustering~\citep{banerjee2005} and optimal transport~\citep{benamou2015iterative}, to name just a few. 

In many modern problems in machine learning, data points (models) may be better understood as functions (nonparametric) rather than finite dimensional vectors (parameters)~\citep{ScholkopfSmola2002,RasmussenWilliams2006,arbel2019maximum,knoblauch2019generalized,songscore,wild2022generalized}. 
This motivates the construction of BDs that operate on functions, along with a corresponding elucidation of their properties and how they mirror the finite dimensional setting. 

\subsection{Contributions} 

We summarise our contributions as follows.

\begin{itemize}
\item We construct a functional Bregman divergence (FBD), but on a Hilbert space $\mathcal{H}$ rather than a Banach space~\cite{bauschke2003bregman} or indeed a Lebesgue measure space $L^r(\nu)$~\cite{J-FunctionalBD-2008}. 
This construction is briefly given in~\cite[Exercise 17.8]{bauschke2020}, however its properties are not discussed at length. 
The Hilbert space affords us a convenient representation of the derivative via the Riesz representation theorem. This is our main distinction from~\cite{J-FunctionalBD-2008}. 

\item We describe some immediate standard properties (Proposition~\ref{prop:basic_properties}) of the functional BD, naturally extending the finite dimensional and pointwise BD~\cite{jones2002general}. 
These properties are nonnegativity and the identity of indiscernibles, convexity in the first argument, linearity in the generator and the three point identity.

\item We show some other more involved properties, and highlight some open  problems in properties which hold for the finite vector Bregman divergences but have not been extended to FBDs. 
For example, our FBD admits a bias-variance decomposition (Lemma~\ref{lem:bregman-decomp}) from which it easily follows that the mean is the minimiser of the FBD with respect to the second argument (Theorem~\ref{thm:min-risk-mean}), but it is not clear whether the converse result holds, as in the finite-dimensional case~\citep{banerjee2005optimality}. 
We can characterise the class of symmetric FBDs, but only for twice differentiable generators (Theorem~\ref{thm:symmetric}). 
We also provide a means of metricising the divergence (Theorem~\ref{thm:metricisation}).

\item Our FBD possesses a natural dually flat structure, under standard Legendre-type conditions on the generator. 
Implications of this include a dual coordinate system, dual divergences and quasi-arithmetic mean as left expected minimisers (Theorem~\ref{thm:reverse-risk-qmean}).

\item From the perspective of estimation of the FBD, a particularly convenient subfamily is a kernelised FBD (k-FBD). Making use of the kernel mean embedding, we describe this family (Proposition~\ref{prop:kfbd}). These k-FBDs are estimatable from data, and include as special cases the popular squared maximum mean discrepancy (MMD), a new deformed squared MMD, and formally a kernelised Kullback-Leibler (KL) divergence. 
The deformed squared MMD can be uniformly sandwiched between a squared MMD, up to easily computable multiplicative constants.

\item We discuss some potential applications of k-FBDs. These include universal and robust estimation, where deformed squared MMD inherits similar properties as the MMD~\citep{cherief2022finite}.
\end{itemize}

\begin{figure}[t]
\centering
\includegraphics[scale=0.8]{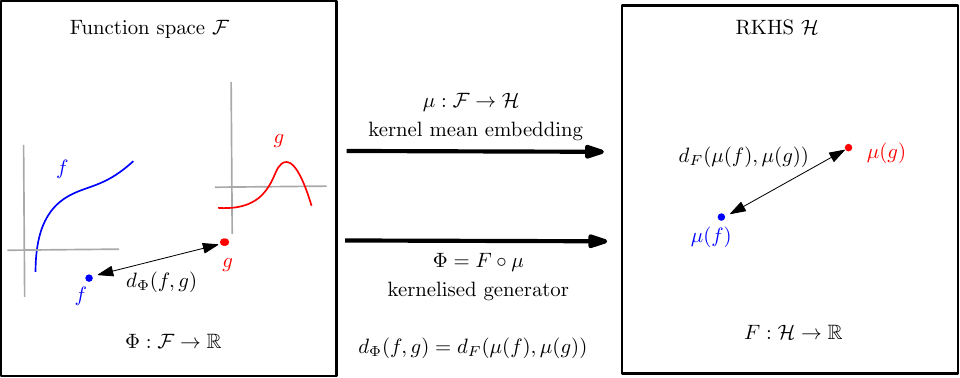}
\caption{Schematic view of a kernelised functional Bregman divergence. Functions $f,g\in\mathcal{F}$ are mapped by the kernel mean embedding $\mu$ into an RKHS $\mathcal{H}$. When the generator on $\mathcal{F}$ factors as $\Phi = F\circ \mu$, the resulting kernelised functional Bregman divergence can be understood directly as a divergence on the kernel mean embeddings, $d_{\Phi}(f,g)=d_{F}(\mu(f),\mu(g))$. See Proposition~\ref{prop:kfbd}.}
\label{fig:kfbd_schematic}
\end{figure}

\newcolumntype{Y}{>{\centering\arraybackslash}X}
\begin{table}[]
    \centering
    \begin{tabularx}{\textwidth}{c|Y|Y|Y|Y}
         & Vector Bregman divergence & Pointwise Bregman divergence & Functional Bregman divergence ($L^s(\nu)$) & Functional Bregman divergence ($\mathcal{H}$) \\
         & \cite{Bregman-1967} & \cite{jones2002general} & \cite{J-FunctionalBD-2008} & This paper \\
         \hline
         Basic properties (see Proposition~\ref{prop:basic_properties}) & \checkmark & \checkmark & \checkmark & \checkmark \\
         Bias-variance decomposition & \checkmark & \checkmark & $\circ$ & \checkmark~\eqref{eq:bias_variance} \\
         Bregman $\implies$ Mean-as-minimiser & \checkmark & \checkmark & \checkmark & \checkmark~\eqref{eq:meanasmin} \\
         Mean-as-minimiser $\implies$ Bregman & \checkmark \cite{banerjee2005optimality} & \checkmark & ? & ? \\
        Characterisation of symmetric divergence & \checkmark & \checkmark & ? & \checkmark Theorem~\ref{thm:symmetric} \\ 
        Metricisation & \checkmark \cite{acharyya2013bregman} & ? & ? & \checkmark~\eqref{eq:gsb} 
    \end{tabularx}
    \caption{Comparison of known properties of our functional Bregman divergences with other types of Bregman divergences. \textbf{Legend:} \checkmark Known and explicitly shown. $\circ$ Implicit or likely true (perhaps given the results shown here) but not shown. ? Open question. We remark that concerning metricisation, the proof of~\cite{acharyya2013bregman} nor our proof extend obviously to the FBD on $L^s(\nu)$, since they rely heavily on an inner product structure, and suspect that the result does not generalise to this setting. The characterisation of the symmetric divergence result in Theorem~\ref{thm:metricisation} relies on a strong hypothesis that the generator is twice differentiable, unlike for the vector BD and pointwise BD results.}
    \label{tab:placeholder}
\end{table}

\section{Background}

\subsection{Generator}
Consider a real Hilbert space $\mathcal{F}$ of functions $f:\mathcal{X} \to \mathbb{R}$, where $\mathcal{X} \subseteq \mathbb{R}^d$. 
Throughout this paper, we will consider functionals  $\Phi:\mathcal{F} \to \mathbb{R}$ called \emph{generators} satisfying some basic conditions. 
Note that these conditions are sometimes stronger than strictly necessary, but for compactness of presentation we use these broader conditions. 
\begin{definition}
    \label{def:generator}
    A function $\Phi:\mathcal{F}\to(-\infty,\infty]$ is called a generator if it is proper, lower semicontinuous (lsc.), strictly convex and Fr\'echet differentiable on $\mathcal{A} = \text{int}(\text{dom} \, \Phi) \neq \emptyset$.
\end{definition}

\subsection{Duality}
\label{sec:Duality}
Some of our optional results concern a dual structure induced by the generator. 
For such dual structures, we require additional conditions on the generator, similar to the finite dimensional vector case~\citep{LegendreType-1967}.
\begin{definition}
\label{def:legendre}
    If a generator $\Phi$ as in Definition~\ref{def:generator} additionally is essentially smooth and essentially strictly convex, then it is said to be of Legendre type. 
\end{definition}
Since a real Hilbert space $\mathcal{F}$ is isomorphic to its own dual $\mathcal{F}^\ast$, the canonical dual pairing may be understood as the inner product itself: by the Riesz representer theorem, for every $g^\ast \in \mathcal{F}^\ast$ there exists a unique $g\in\mathcal{F}$ such that $g^\ast(f) = \langle g, f \rangle_{\mathcal{F}}$ for all $f \in \mathcal{F}$.
We therefore identify the dual space of $\mathcal{F}$ as $\mathcal{F}$ itself rather than $\mathcal{F}^\ast$, and directly consider $g$ rather than the corresponding $g^\ast$. 
This leads to a simplified presentation of a convex conjugate involving only an inner product rather than a dual pairing. 
\begin{definition}[Definition 13.1 of~\cite{bauschke2020}]
\label{def:conjujgate}
The convex conjugate $\Phi^\ast:\mathcal{F}\to(-\infty, \infty]$ of (a not necessarily Legendre type generator)  $\Phi:\mathcal{F}\to(-\infty,\infty]$ is defined by
\begin{align*}
    \Phi^\ast(g^\ast) = \sup_{f \in \mathcal{F}} \langle g^\ast, f \rangle_{\mathcal{F}} - \Phi(f)
\end{align*}
and the biconjugate of $\Phi$ is $\Phi^{\ast\ast}=(\Phi^\ast)^\ast$.
\end{definition}
We need an appropriate notion of gradients $\nabla \Phi$ of functionals $\Phi$ on $\mathcal{F}$:
We use the Fr\'echet derivative; see Appendix~\ref{app:frechet} for more detail.
When $\Phi$ is a Legendre type generator, a maximiser to the problem in Definition~\ref{def:conjujgate} exists and it is unique. 
The first order gradient condition (Proposition 17.4 of~\cite{bauschke2020}) together with the Riesz representer theorem says that the maximiser $f'$ satisfies
$D\Phi[f';h] = \langle \nabla \Phi(f') , h \rangle_{\mathcal{F}} = \langle g^\ast, h\rangle_{\mathcal{F}}$ for test functions $h \in \mathcal{F}$. 
Convex conjugate functionals satisfy a Fenchel-Young inequality. 
\begin{proposition}[Proposition 13.15 and (implication of) Proposition 16.10 of~\cite{bauschke2020}]
    Let $\Phi:\mathcal{F} \to (\infty,\infty]$ be proper. Then for all $f,g \in \mathcal{H}$
    $$\Phi(f) + \Phi^\ast(g) \geq \langle f, g \rangle_{\mathcal{F}}.$$
    Moreover, suppose $\Phi$ is Fr\'echet differentiable. 
    Then equality is obtained if and only if $g = \nabla \Phi(f)$, which implies $f = \nabla \Phi^\ast(g)$.
\end{proposition}
We remark that it is possible to define a functional Fenchel-Young divergence through the gap $\Phi(f) + \Phi^\ast(g) - \langle f, g\rangle_{\mathcal{F}} \geq 0$, extending the finite vector case~\citep{blondel2020learning}, but we do not further pursue that direction here.
BDs can be expressed using mixed dual parameterisations as equivalent Fenchel-Young divergences which are used in ML in the scalar case as Fenchel-Young losses.
Convex conjugates also satisfy a biconjugate property, known as a Fenchel-Moreau theorem. 
\begin{proposition}[Theorem 13.37 of~\cite{bauschke2020}]
    Let $\Phi:\mathcal{F}\to(-\infty,\infty]$ be proper. Then $\Phi$ is lower semicontinuous and convex if and only if $\Phi=(\Phi^\ast)^\ast$. In this case, $\Phi^\ast$ is proper as well.
\end{proposition}

\subsection{Functional Bregman Divergence}

We now introduce the functional Bregman divergence (FBD). 
The definition below actually extends to Gateaux derivatives, but here we consider only Fr\'echet derivatives (see Appendix~\ref{app:frechet}). 
\begin{definition}[Implied by Exercise 17.8 of~\cite{bauschke2020}]
\label{def:fbd}
    Let $\Phi:\mathcal{F} \to (-\infty,\infty]$ be a a generator (see Definition~\ref{def:generator}) on $\mathcal{F}$.  
    The \emph{functional Bregman divergence} (FBD) $d_\Phi:\mathcal{A}\times\mathcal{A}\to(-\infty,\infty)$ generated by $\Phi$ is defined by
    \begin{align*}
        d_\Phi(f,g) = \Phi(f) - \Phi(g) - \langle \nabla \Phi(g),  f-g \rangle_{\mathcal{F}},
    \end{align*}
    where $\nabla \Phi(g)$ is the Fr\'echet gradient of $\Phi$ at $g$.
\end{definition}
We remark that it is in principle also possible to construct a Jensen-Bregman divergence (as in the finite vector case~\citep{nielsen2011burbea}), a type of symmetrisation of the Bregman divergence. 
While we do not investigate functional Jensen-Bregman divergences here, we later consider a symmetricisation and metricisation of the FBD in Theorem~\ref{thm:metricisation}.

\subsection{RKHS, MMD and KME}
\paragraph{RKHS.} 
We make use of a reproducing kernel Hilbert space (RKHS) $\mathcal{H}$, different from the Hilbert space $\mathcal{F}$.
An RKHS $\mathcal{H}$ on $\mathcal{X} \subseteq \mathbb{R}^d$ is a Hilbert space of functions $f:(\mathcal{X}\subseteq \mathbb{R}^d) \to \mathbb{R}$ characterised by a positive semidefinite kernel $k:\mathcal{X}\times \mathcal{X} \to \mathbb{R}$ such that the reproducing / evaluation property holds: $\langle f, k(x, \cdot) \rangle_{\mathcal{H}} = f(x)$, for any $f \in \mathcal{H}$ and $x \in \mathcal{X}$. (Equivalently, a Hilbert space $\mathcal{H}$ on $\mathcal{X}$ is an RKHS if for every $x \in \mathcal{X}$, the evaluation functional $L_x:\mathcal{H} \to \mathbb{R}$ defined by $L_x(f) = f(x)$ is continuous in $f \in \mathcal{H}$ (i.e. $L_x$ is a bounded linear operator)).

\paragraph{Maximum mean discrepancy and kernel mean embedding.}
The integral probability metric
\begin{align*}
    \mathrm{MMD}_k(f,g) = \sup_{f \in \mathcal{H}\, \Vert f \Vert_{\mathcal{H}}\leq 1} \mathbb{E}_{x\sim f, y \sim g} f(x) - f(y) 
\end{align*}
is called the \emph{maximum mean discrepancy} (MMD)~\citep{gretton2006kernel, gretton2012kernel, sriperumbudur2010hilbert}.
Using the Cauchy Schwarz inequality and the reproducing property, one can show that the squared MMD can be written as
\begin{align*}
    \text{MMD}^2_k(f,g) = \Vert \mu(f) - \mu(g) \Vert_{\mathcal{H}}^2 = \mathbb{E} [k(x, x') - 2 k(x,y) + k(y,y')], \qquad x,x' \stackrel{iid}{\sim} f, \quad y,y' \stackrel{iid}{\sim} g, \numberthis \label{eq:mmd}
\end{align*}
where $\mu:\mathcal{F}\to\mathcal{H}$ and $\mu(f) \in \mathcal{H}$ is the kernel mean embedding (KME) of $f$, given by $\mu(f) = \mathbb{E}_{x\sim p} [k(x,\cdot)]$, provided that $\Vert \mu(f) \Vert_{\mathcal{H}} < \infty$.

\section{Properties of the functional Bregman divergence}
For the remainder of the paper, we defer all long proofs to Appendix~\ref{app:proofs}.
\subsection{Basic properties}
We first show some basic properties concerning the FBD, which naturally extend the BD in finite dimensions. 
\begin{restatable}{proposition}{propbasic}
\label{prop:basic_properties}
The FBD $d_{\Phi}$ in Definition~\ref{def:fbd} satisfies the following basic properties:

\begin{enumerate}
    \item \textbf{Nonnegativity and identity of indiscernibles.} For all $f,g \in \mathcal{A}$ the divergence is nonnegative $d_\Phi[f,g] \geq 0$ and $d_\Phi[f,g]= 0$ if and only if $f=g$.
    \item \textbf{Convex in the first argument.} For fixed $q$, the function $d_{\Phi}(\cdot,q):\mathcal{F} \to [0,\infty)$ is strictly convex.
    \item \textbf{Linearity in generator.} For any two generators (see Definition~\ref{def:fbd})  $\Phi_1:\mathcal{F} \to \mathbb{R}$ and $\Phi_2:\mathcal{F} \to \mathbb{R}$ and $\lambda \geq 0$, we have that $d_{\Phi_1 + \lambda \Phi_2} = d_{\Phi_1} + \lambda d_{\Phi_2} $. 
    \item \textbf{Three point identity.} $ d_{\Phi}(f,g) = d_{\Phi}(f,h) + d_{\Phi}(h,g) - \langle \nabla\Phi(g) - \nabla\Phi(h), f - h  \rangle_{\mathcal{F}}$.
\end{enumerate}
\end{restatable}
The FBD also has a dual divergence, making use of the Legendre duality introduced in \S~\ref{sec:Duality}. 
This duality follows the same style as in the finite vector Bregman case. 
\begin{theorem}
\label{thm:dual-bregman}
Let $\Phi$ be a Legendre-type generator. 
Then for all $f,g\in \mathcal{A}$,
\[
d_\Phi(g,f)\;=\; d_{\Phi^\ast}\bigl(\nabla\Phi(f),\,\nabla\Phi(g)\bigr).
\]
\end{theorem}

\begin{proof}
Fix $f,g\in \mathcal{A}$ and write $u=\nabla\Phi(f)$, $v=\nabla\Phi(g)$.
By the Fenchel--Young equality at $f$ and $g$,
$\Phi(f)=\langle f,u\rangle_{\mathcal F}-\Phi^\ast(u)$ and $\Phi(g)=\langle g,v\rangle_{\mathcal F}-\Phi^\ast(v)$.
Substitute these into $d_\Phi(g,f)=\Phi(g)-\Phi(f)-\langle \nabla \Phi(f),g-f\rangle_{\mathcal F}$ and simplify to get
$d_\Phi(g,f)=\Phi^\ast(u)-\Phi^\ast(v)-\langle g,u-v\rangle_{\mathcal F}$.
Using $\nabla\Phi^\ast(v)=g$, this is exactly $d_{\Phi^\ast}(u,v)$.
\end{proof}

\subsection{Bias-variance decomposition and the mean-as-minimiser}
In order to prove the mean-as-minimiser property, we take an alternative route to~\cite{J-FunctionalBD-2008}, who proceed with calculus. 
We do so through the bias-variance decomposition, which is an interesting property in its own right. 
In order to formulate the idea of a mean, we first need a notion of random functions and integration of such functions. 
For this, as in~\cite{J-FunctionalBD-2008}, we use the Bochner mean (see Appendix~\ref{app:bochner} for a very brief overview). 
\begin{lemma}
\label{lem:bregman-decomp}
Let $\mathcal{U}\subset \mathcal F$ be a nonempty convex set and let $\Phi$ be a generator. 
Let $F$ be an $\mathcal F$-valued random element such that $\mathbb P(F\in \mathcal{U})=1$ and
$\mathbb E\|F\|_{\mathcal F}<\infty$. Assume $\mathbb E|\Phi(F)|<\infty$ and that the Bochner mean
$\bar f=\mathbb E[F]$ belongs to $\mathcal{U}$.
Then for every $g\in \mathcal{U}$ for which $\mathbb E[d_\Phi(F,g)]$ is finite,
\begin{equation}
\mathbb E\bigl[d_\Phi(F,g)\bigr]
=
\mathbb E\bigl[d_\Phi(F,\bar f)\bigr] + d_\Phi(\bar f,g). \label{eq:bias_variance}
\end{equation}
\end{lemma}

\begin{proof}
Expand the difference using the definition of $d_\Phi$:
\[
\mathbb E[d_\Phi(F,g)]-\mathbb E[d_\Phi(F,\bar f)]
=
-\Phi(g)+\Phi(\bar f)
-\big\langle \nabla\Phi(g),\,\mathbb E[F]-g\big\rangle_{\mathcal F}
+\big\langle \nabla\Phi(\bar f),\,\mathbb E[F]-\bar f\big\rangle_{\mathcal F}.
\]
Since $\mathbb E[F]=\bar f$, the last inner product vanishes and the remaining terms equal
$\Phi(\bar f)-\Phi(g)-\langle \nabla\Phi(g),\,\bar f-g\rangle_{\mathcal F}=d_\Phi(\bar f,g)$.
\end{proof}
With the bias-variance decomposition, showing the mean-as-minimiser property becomes direct.
\begin{theorem}
\label{thm:min-risk-mean}
Let $\mathcal{U}\subset \mathcal F$ be a nonempty open convex set and let $\Phi$ be a generator. 
Let $F$ be an $\mathcal F$-valued random element such that
$\mathbb P(F\in \mathcal{U})=1$, $\mathbb E\|F\|_{\mathcal F}<\infty$, $\mathbb E|\Phi(F)|<\infty$,
and assume that $\bar f=\mathbb E[F]\in \mathcal{U}$.
Then $\bar f$ is the unique minimiser of $g\mapsto \mathbb E[d_\Phi(F,g)]$ over $\mathcal{U}$, that is
\begin{equation}
\bar f = \argmin_{g\in \mathcal{U}}\ \mathbb E\bigl[d_\Phi(F,g)\bigr]. \label{eq:meanasmin}
\end{equation}
\end{theorem}

\begin{proof}
By Lemma~\ref{lem:bregman-decomp}, for all $g\in \mathcal{U}$,
\[
\mathbb E[d_\Phi(F,g)] = \mathbb E[d_\Phi(F,\bar f)] + d_\Phi(\bar f,g).
\]
The first term does not depend on $g$. The second term is nonnegative and equals $0$
if and only if $g=\bar f$. Therefore $\bar f$ minimises the risk, and the minimiser is unique.
\end{proof}
Combining the mean-as-minimiser property together with the connection between the dual divergence and the self-duality of the Hilbert space, we can also understand minimisers of expected FBDs in the first argument as \emph{quasi-arithmetic means}.  
This result also extends the finite vector Bregman divergence case. 
\begin{restatable}{theorem}{quasimean}
\label{thm:reverse-risk-qmean}
Let $\mathcal{U}\subset \mathcal F$ be a nonempty open convex set and let $\Phi$ be a Legendre-type generator. 
Let $F$ be an $\mathcal F$-valued random element such that $\mathbb P(F\in \mathcal{U})=1$, and assume that
$Z=\nabla\Phi(F)$ is Bochner integrable with Bochner mean $\bar z=\mathbb E[Z]\in V$.
Then the problem
$
\argmin_{g\in \mathcal{U}}\; \mathbb E\bigl[d_\Phi(g,F)\bigr]
$
admits a unique minimiser on $\mathcal{U}$, given by
\[
g^\star \;=\; \nabla\Phi^\ast(\bar z)\;=\;(\nabla\Phi)^{-1}\!\left(\mathbb E[\nabla\Phi(F)]\right) = \argmin_{g\in \mathcal{U}}\; \mathbb E\bigl[d_\Phi(g,F)\bigr].
\]
In particular, $g^\star$ is the quasi-arithmetic mean of $F$ associated with the transform $\nabla\Phi$.
\end{restatable}

\subsection{Characterisation of symmetric FBDs}
It is well-known that in the finite-vector Bregman divergence setting, the only Bregman divergence which is symmetric is the squared Mahalanobis distance. 
This result does not hold for FBDs in $L^r(\nu)$ space, as shown by the squared-bias example of~\cite{J-FunctionalBD-2008} (example 2.1.2). 
However, for Hilbert spaces $\mathcal{H}$, such a characterisation does indeed hold, for twice differentiable generators. 
\begin{restatable}{theorem}{symmetric}
\label{thm:symmetric}
Let $\Phi$ be a twice continuously Fréchet differentiable Legendre-type generator on $\mathcal{A}$. 
Let $d_\Phi$ be the functional Bregman divergence as defined in Definition~\ref{def:fbd}. 
Then the following are equivalent:
\begin{enumerate}
\item The FBD is symmetric, that is, $d_\Phi(f,g)=d_\Phi(g,f)$ for all $f,g \in \mathcal{A}$.
\item There exist a bounded self-adjoint strictly positive operator $T : \mathcal{F} \to \mathcal{F}$, $b \in \mathcal{F}$, and $c \in \mathbb{R}$ such that
\[
\Phi(f)=\frac12\langle f,Tf\rangle_{\mathcal{F}}+\langle b,f\rangle_{\mathcal{F}}+c, \qquad f\in \mathcal{A}.
\]
\item There exists a bounded self-adjoint strictly positive operator $T : \mathcal{F} \to \mathcal{F}$ such that
\[
d_\Phi(f,g)=\frac12\langle f-g,T(f-g)\rangle_{\mathcal{F}}, \qquad f,g\in \mathcal{A}.
\]
\end{enumerate}
\end{restatable}

\subsection{Metricisation of FBD}
One approach to symmetrising and metricising a Bregman divergence in finite dimensions is given by~\cite{acharyya2013bregman}, whose basic construct we generalise for FBD. 
In order to state and prove our result, we require a notion of a positive semidefinite operator on a Hilbert space. 
Construction of positive definite and positive semidefinite adjoint operators (via L\"owner ordering) on $\mathcal{F}^2$ are discussed in Example~13.18 of~\cite{bauschke2020}. 
\begin{theorem}
\label{thm:metricisation}
    Let $A$ and $B$ be bounded self-adjoint positive definite operators on $\mathcal{A}$. 
    Define the symmetrised square-root divergence 
    \begin{equation}
    \label{eq:gsb}
    d_\Phi^{gsb}(f, g) = d_\Phi(f,g) + d_\Phi(g,f) + \frac{1}{2}\Vert f - g \Vert_A^2 + \frac{1}{2} \Vert \nabla \Phi(f) - \nabla \Phi(g) \Vert_B^2,
    \end{equation}
    on the non-empty interior domain of $\Phi$, 
    where $\Vert \cdot \Vert_B^2 = \langle \cdot, B \cdot \rangle_{\mathcal{F}} $ and $\Vert \cdot \Vert_A^2 = \langle \cdot, A \cdot \rangle_{\mathcal{F}}$.
    Define the product Hilbert space $\mathcal{F}^2$ with inner product $\langle (f_1, f_2), (g_1, g_2) \rangle_{\mathcal{F}^2} = \langle f_1, g_1 \rangle_{\mathcal{F}} + \langle f_2, g_2 \rangle_{\mathcal{F}}$.
    Define a self-adjoint operator $M:\mathcal{F}^2 \to \mathcal{F}^2$ by 
    \begin{align*}
        M\big((f,g)\big) = (Af + g, f + Bg)
    \end{align*}
    Then $\sqrt{d_\Phi^{gsb}(f, g)}$ is a metric on $\mathcal{A}^2$ if $M$
    is PSD on $\mathcal{A}^2$. 
    A sufficient condition on $A$ and $B$ can be derived from the Schur complement of $M$ ($B\succeq A^{-1}$ and $A$ is strongly positive, that is, $\langle f, Af \rangle_{\mathcal{F}} \geq \alpha \Vert f \Vert_{\mathcal{F}}$ for some $\alpha > 0$). 
\end{theorem}
\begin{proof}
    Define a block operator $M$ as above. 
    Note that $M$ is bounded: $\Vert M((f,g)) \Vert_{\mathcal{F}^2} \leq \Vert Af \Vert_{\mathcal{F}} + \Vert g \Vert_{\mathcal{F}} + \Vert f \Vert_{\mathcal{F}} + \Vert Bg \Vert_{\mathcal{F}} \leq (\Vert A \Vert + 1) \Vert f \Vert_{\mathcal{F}} + (\Vert B \Vert + 1) \Vert g \Vert_{\mathcal{F}}$ and self-adjoint. 
    Define an embedding $b(f)$ by 
    \begin{align*} 
        b(f) = \begin{pmatrix}
            f \\ \nabla \Phi(f) 
        \end{pmatrix}  \in \mathcal{F}^2.
    \end{align*}
    Then the generalised symmetric FBD can be represented as a Mahalanobis squared distance on the embeddings,
    \begin{align*}
        d_\Phi^{gsb}(f, g) &= \frac{1}{2}\big\langle b(f) - b(g), \, M  \big(b(f) - b(g)\big) \big\rangle_{\mathcal{F}^2}.
    \end{align*}
    Hence if $M$ is positive semidefinite, $\sqrt{d_\Phi^{gsb}}$ satisfies nonnegativity, symmetry and triangle inequality. 
    For the identity of indiscernibles, note that all three terms in~\eqref{eq:gsb} are nonnegative, and the first is strictly positive (since $\Phi$ is strictly convex) unless $f = g$. 
\end{proof}
We remark that our result applies to the finite dimensional case without the need for the use of tools from conditionally positive semidefinite functions~\citep{schoenberg1938metric}, and in particular covers the matrix extension~\citep[Theorem 4]{acharyya2013bregman}. 
Metricising the FBD allows for a host of opportunities to build non-Euclidean metric algorithms and applications. 
For example, one can use the triangle inequality to accelerate the k-means algorithm~\citep{elkan2003using}, by noting that in order to assign points to cluster centres, it often suffices to lower or upper bound distances to cluster centres.

\subsection{Kernelised functional Bregman divergences}
\label{sec:k-FBD}
Let $k:X \times X \to\mathbb R$ be a positive semidefinite kernel with associated RKHS $\mathcal H$.
Assume $k$ is uniformly bounded, so that the kernel mean embedding $\mu:\mathcal{F}\to\mathcal{H}$ defined by 
$
\mu(f)=\int_\Omega k(x,\cdot)\,f(x)\,d\nu(x)\in\mathcal H
$
is well-defined for all integrable $f\in \mathcal{F}$.
Then $\mu$ is bounded linear, and its adjoint
$
\mu^*:\mathcal H\to \mathcal{F}
$
is well-defined. 
The following observation allows us to construct kernelised FBDs, which we will find easy to estimate. 
\begin{restatable}{proposition}{kFBD}
\label{prop:kfbd}
    Consider a Bregman functional $\Phi:\mathcal{F}\to\mathbb{R}$ of the form $\Phi(f) = F(\mu(f))$, where $\mu(f)$ is the kernel mean embedding of $f$, and $F:\mathcal{H}\to\mathcal{R}$ is a generator.
    Then for all $f, g \in \mathcal{A}$,
    $$ d_{\Phi}(f,g) = d_F\big( \mu(f), \mu(g) \big).$$
    Moreover, suppose $\mu$ is injective on $\mathcal{A}$. 
    Then $\Phi = F \circ \mu$ is a generator (as in Definition~\ref{def:generator}).
\end{restatable}
We similarly inherit relevant results regarding duality when $F$ is additionally of Legendre type (Definition~\ref{def:legendre}). 

\paragraph{Noising interpretation.} If the kernel $k$ is stationary, that is, there exists some $\kappa$ such that $k(x,y) = \kappa(x-y)$, and moreover if $\kappa$ is itself a probability density function with respect to some measure $\nu$, then the kernel mean embedding
\begin{equation}\mu(f)(x) = \int \kappa(x-y) f(y) \nu(dy) \label{eq:noising}
\end{equation}
may be interpreted as the density obtained by adding independent noise following $\kappa$ to a random variable following $f$. 
Hence this k-FBD is a divergence generated by $F$ on the noised densities. 

\subsection{Deformed squared MMD}
\label{sec:deformed}
We detail several special cases of generator $F$ in Appendix~\ref{app:special_cases}. 
In the main body, we specialise on a particular k-FBD, generalising the squared MMD.
\begin{definition}
\label{def:deformed}
Let $\phi:[0,\infty)\to \mathbb R$ be twice continuously Fr\'echet differentiable on $(0,\infty)$, such that $\phi'(r) >0$ for $r >0$ and $\phi''(r) > 0$ for $r > 0$. A functional $\Phi:\mathcal{F} \to \mathbb R$ of the form $\Phi(f)=\phi(\|\mu(f)\|_{\mathcal H})$ is called a deformed generator.  
\end{definition}
Since $\mu$ is linear and $u\mapsto \phi(\|u\|_{\mathcal H})$ is convex on $\mathcal H$, the deformed generator $\Phi$ is convex. 
For $\|\mu(f)\|_{\mathcal H}>0$, the Fr\'echet derivative at $f$ acting on a function $a \in \mathcal{F}$ satisfies (see Examples~\ref{example:frechet_radial} and~\ref{example:frechet_radial_composed})
\[
D\Phi[f;a]
=
\Bigg\langle
\frac{\phi'(\|\mu(f)\|_{\mathcal H})}{\|\mu(f)\|_{\mathcal H}}\,
\mu(f), \mu(a)
\Bigg\rangle_{\mathcal{H}}.
\]
For $\|\mu(f)\|_{\mathcal H}>0$, the Hessian operator
$D^2\Phi[f;\cdot,\cdot]: \mathcal{F} \times \mathcal{F} \to \mathbb{R}$
satisfies (see Examples~\ref{example:frechet_hessian})
\[
D^2\Phi[f;a,b]
=
\Bigg\langle \Bigg(
\frac{\phi'(\|\mu(f)\|_{\mathcal H})}{\|\mu(f)\|_{\mathcal H}}\,I_{\mathcal H}
+
\left(
\phi''(\|\mu(f)\|_{\mathcal H})
-
\frac{\phi'(\|\mu(f)\|_{\mathcal H})}{\|\mu(f)\|_{\mathcal H}}
\right)
\frac{\mu(f)\otimes \mu(f)}{\|\mu(f)\|_{\mathcal H}^{2}} \Bigg)
\mu(a), \, \mu(b) \Bigg\rangle_{\mathcal{H}}.
\]
We may write a spectral decomposition of the Hessian operator by identifying orthogonal projections and associated eigenvalues,
\[
P^{\parallel}_{\mu(f)}
=
\frac{\mu(f)\otimes \mu(f)}{\|\mu(f)\|_{\mathcal H}^{2}},
\qquad
P^{\perp}_{\mu(f)}
=
I_{\mathcal H}-P^{\parallel}_{\mu(f)},  \qquad\lambda^\parallel(r)=\phi''(r), \qquad \lambda^\perp(r)=\phi'(r)/r.
\]
We verify that the Hessian is indeed positive definite, under the assumption that $\mu$ is injective. 
Finally, the k-FBD itself is given by
\begin{align*}
    d_\Phi(f,g)
&=
\phi\!\left(\|\mu(f)\|_{\mathcal H}\right)
-
\phi\!\left(\|\mu(g)\|_{\mathcal H}\right)
-
\frac{\phi'\!\left(\|\mu(g)\|_{\mathcal H}\right)}{\|\mu(g)\|_{\mathcal H}}
\,
\big\langle
\mu(g),
\mu(f)-\mu(g)
\big\rangle_{\mathcal H}
\end{align*}
The special case of $\phi(r)=r^2$ recovers the squared maximum mean discrepancy on probability distributions,
\begin{align*}
    \Phi(p) = \Vert \mu(p) \Vert_\mathcal{H}^2 \implies d_{\Phi}(p,q) = \mathrm{MMD}_k^2(p,q). \numberthis \label{eq:mmd_special_case} 
\end{align*}
We note also that the squared MMD case corresponds with a self-dual generator $F$, and belongs to the class of symmetric divergences characterised by Theorem~\ref{thm:symmetric}, with $T = 2\mu^\ast \mu$, where $\mu^\ast$ is the adjoint of $\mu$.

\paragraph{Comparison with squared MMD on a KME ball.}
Fix $R>0$, and define the KME ball
\[
\mathcal P_R
\;=\;
\{\,p\in \mathcal{A} \;:\;\|\mu(p)\|_{\mathcal H}\le R\,\}.
\]
For $r>0$, the two eigenvalues of the Hessian of the map $u\mapsto \phi(\|u\|_{\mathcal H})$ are
$\lambda^\perp(r)=\phi'(r)/r$ and $\lambda^\parallel(r)=\phi''(r)$. Since $\phi\in C^2([0,\infty))$ and $\phi'(0)=0$, the quotient $\phi'(r)/r$ extends continuously to $r=0$ with value $\phi''(0)$. We therefore define
\[
\lambda^\perp(0)=
\lambda^\parallel(0)=\phi''(0), \quad
L_\phi(R)
\;=\;
\sup_{0\le r\le R}\max\{\lambda^\parallel(r),\lambda^\perp(r)\},
\quad
m_\phi(R)
\;=\;
\inf_{0\le r\le R}\min\{\lambda^\parallel(r),\lambda^\perp(r)\}.
\]
for $R > 0$.
These quantities control the comparison of the FBD generated by $\Phi$ with squared MMD, as the result below shows.
\begin{table}[t]
\centering
\begin{tabular}{c||c|c||c|c}
$\phi(r)$
& $\lambda_r^\perp$
& $\lambda_r^\parallel$
& $L_\phi(R)$
& $m_\phi(R)$
\\ \hline\hline

$r^2$
& $2$
& $2$
& $2$
& $2$
\\[0.3em]

$e^r-1-r$
& $\dfrac{e^r-1}{r}$ for $r>0$, with value $1$ at $r=0$
& $e^r$
& $e^R$
& $1$
\\[0.8em]

$2\log\cosh r$
& $2\dfrac{\tanh r}{r}$ for $r>0$, with value $2$ at $r=0$
& $2\operatorname{sech}^2 r$
& $2$
& $2\operatorname{sech}^2(R)$
\\[0.8em]

$\sqrt{1+r^2}-1$
& $(1+r^2)^{-1/2}$
& $(1+r^2)^{-3/2}$
& $1$
& $(1+R^2)^{-3/2}$
\\[0.8em]

$r^2+\lambda r^4 \;(\lambda\ge 0)$
& $2+4\lambda r^2$
& $2+12\lambda r^2$
& $2+12\lambda R^2$
& $2$
\\[0.3em]

$r^p \;(p>2)$
& $p\,r^{p-2}$ for $r>0$, with value $0$ at $r=0$
& $p(p-1)\,r^{p-2}$
& $p(p-1)R^{p-2}$
& $0$
\\ \hline\hline

$r^p \;(1<p<2)$
& \multicolumn{4}{c}{not covered here, since $\phi\notin C^2([0,\infty))$}
\\[0.3em]

$e^r$
& \multicolumn{4}{c}{not covered here, since $\phi'(0)\neq 0$}
\\[0.3em]

$r$
& \multicolumn{4}{c}{not covered here, since $\phi'(0)\neq 0$}
\end{tabular}
\caption{\label{tab:examples}
Examples of radial generators $\phi$ and the corresponding Hessian eigenvalues on $\mathcal H$, together with the resulting comparison constants on the KME ball $\mathcal P_R$.
}
\end{table}

\begin{restatable}{theorem}{sandwich}
\label{thm:radial-bounds}
Consider $\phi$ and $\Phi$ as in Definition~\ref{def:deformed} and satisfy $\phi'(0)=0$, and assume that $\phi$ nondecreasing on $[0,\infty)$. 
Then, for every $R>0$ and every $p,q\in \mathcal P_R$,
\[
\frac{m_\phi(R)}{2}\,\mathrm{MMD}_k^2(p,q)
\;\le\;
d_{\Phi}(p,q)
\;\le\;
\frac{L_\phi(R)}{2}\,\mathrm{MMD}_k^2(p,q),
\]
where $\mathrm{MMD}^2_k$ is as in~\eqref{eq:mmd}. 
In particular, if $m_\phi(R)>0$ and $L_\phi(R)<\infty$, then $d_{\Phi}$ is uniformly comparable to squared MMD on $\mathcal P_R$.
\end{restatable}
We show some special cases of closed-form values of $m_\phi(R)$ and $L_{\phi}(R)$ in Table~\ref{tab:examples}. 
To obtain similar sandwich bounds in finite dimensions, one typically needs upper and lower bounds on the eigenvalues of the Hessian of the generator. 
This could be achieved by taking a local subset, but this is problematic for a uniform bound. 
In the kernel setting, we get the radius $R$ ``for free'' using a bounded kernel (e.g. using a squared  exponential kernel). 
This yields uniform bounds on the eigenvalues.

\subsection{Other special cases}
Formally, let $G:\mathcal{H}\to\mathcal{H}$ and consider generators $F:\mathcal{H}\to\mathbb{R}$ of the form
\begin{align*}
    F\big(\mu(p)\big) &= \langle \mu(p), G\big(\mu(p)\big) \rangle_{\mathcal{H}} = \mathbb{E}_{x\sim p} \big[ G\big(\mu(p) \big) (x) \big] ,
\end{align*}
where the second equality is by the reproducing property of the kernel mean embedding. 
The first Fr\'echet gradient of $F$ is
$
    \nabla F(\mu) = G(\mu) +  \nabla G(\mu)^\ast \mu.
$
Hence we may further simplify the k-FBD, using the reproducing property of the kernel mean embedding, finding
\begin{align*}
     d_{F}\big( \mu(p), \mu(q) \big) 
    &= \mathbb{E}_{x\sim p}\!\left[ G(\mu(p))(x) \right]
       - \mathbb{E}_{x\sim p}\!\left[ G(\mu(q))(x) \right]
       - \mathbb{E}_{y\sim q}\!\left[ \nabla G(\mu(q))\big(\mu(p)-\mu(q)\big)(y) \right]. \numberthis \label{eq:special_case}
\end{align*}
Formally, some interesting special cases are
\begin{itemize}
    \item When $G$ is the identity, we recover the squared MMD, where $\nabla G(\mu) = \text{Id}$ and $\nabla^2 G (\mu) = 0$.
    \item When $G(\mu) = \int_{\mathcal{X}} \sigma\big( \mu(x) \big) k(x, \cdot) \, d\nu(x)$ for some $\nu$-integrable $\sigma:\mathbb{R}\to\mathbb{R}$ and $k$ is the bounded reproducing kernel of $\mathcal{H}$, we find, using the reproducing property,
    $$F(\mu) = \int \sigma\big( \mu(x) \big) \langle \mu, k(x,\cdot)  \rangle_{\mathcal{H}} \, d\nu(x) = \int_{\mathcal{X}} \sigma\big( \mu(x) \big) \mu(x) \, d\nu(x).$$ Hence we may recover a pointwise generator in $L^2(\nu)$, but on the kernelised density rather than the density itself. 
    \begin{itemize}
        \item When $G(\mu) = \sigma(\Vert \mu \Vert_{\mathcal{H}}) \mu,$ we recover the deformed squared MMD, as studied in \S~\ref{sec:deformed}, up to a reparameterisation $\sigma(\Vert \mu \Vert_{\mathcal{H}})\Vert \mu \Vert_{\mathcal{H}}^2 = \phi\big(\Vert \mu \Vert_{\mathcal{H}}\big).$
        \item Another example is a kernelised negative entropy generator, obtained with $\sigma = \log$.
    \end{itemize}
\end{itemize}
\section{Applications}

\subsection{Divergence estimation}
Given samples $x_1,\ldots,x_n \sim p$ and $y_1, \ldots y_m \sim q$ together with corresponding empirical distributions $\hat{p}_n$ and $\hat{q}_m$, we may form a naive Monte Carlo estimate of~\eqref{eq:special_case},
\begin{align*}
    \hat{d}_{\Phi}(p,q) = \frac{1}{n}\sum_{i=1}^n \Big( G\big( \mu(\hat{p}_n)\big)(x_i) - G\big( \mu(\hat{q}_m)\big)(x_i)  \Big) - \frac{1}{m}\sum_{j=1}^m \nabla G\big( \mu(\hat{q}_m) \big)\big( \mu(\hat{p}_n) - \mu(\hat{q}_m)\big)(y_j).
\end{align*}
Such estimators, if well-behaved, may in future be used to form hypothesis tests, noting that the estimator is biased and similar tests on biased statistics are performed on the squared MMD~\citep{gretton2006kernel}.

\subsection{Generative modelling}
Generative models can be trained by minimising an empirical divergence between samples generated by a deep neural network and a dataset of samples from a target distribution.
The squared MMD has been used in this capacity to build generative models for data such as images~\citep{li2015generative,li2017mmd}. 
The key property of the squared MMD here is that it can be estimated from data samples and generated samples. 
In principle, the k-FBD affords us alternative divergences to the squared MMD, opening the door for further model exploration.

\subsection{Universal and robust estimation}
With the sandwich inequalities in Theorem~\ref{thm:radial-bounds}, one can obtain a useful inequality for universal and robust estimation. 
Let $(p_\theta)_{\theta\in\Theta}$ be a statistical model on $\mathcal{X}$. 
Let $p_0$ be the density of the true data-generating distribution, which need not belong to the model.
Given observations $X_1,\dots,X_n \overset{iid}{\sim} p_0$, write
$\widehat p_n = n^{-1}\sum_{i=1}^n \delta_{X_i}$ for the empirical density.
Let $\sqrt{D}$ be a metric on a class of probability measures containing
$\{p_\theta:\theta\in\Theta\}\cup\{p_0,\widehat p_n\}$, and let $\sqrt{D'}$ be a nonnegative discrepancy on the same class such that
$\sqrt{m/2}\,\sqrt{D(P,Q)} \le \sqrt{D'(P,Q)} \le \sqrt{L/2}\,\sqrt{D(P,Q)}$
for all $P,Q$ in that class, for some constants $m,L>0$.
Define the minimum-divergence estimator by
\begin{equation}
\widehat\theta_n \in \operatorname*{argmin}_{\theta\in\Theta} \sqrt{D'(p_\theta,\widehat p_n)} = {\argmin}_{\theta\in\Theta} D'(p_\theta,\widehat p_n). \label{eq:estimator}
\end{equation}

\begin{restatable}{proposition}{universalestimationgeneral}
\label{prop:universal-metric}
For any $\theta\in\Theta$,
\[
\sqrt{D'(p_{\widehat\theta_n},p_0)}
\le
\frac{L}{\sqrt{2m}}\,\sqrt{D(p_\theta,p_0)}
+
\left(
\sqrt{\frac{L}{2}}+\frac{L}{\sqrt{2m}}
\right)\sqrt{D(\widehat p_n,p_0)}.
\]
Consequently,
\[
\mathbb E\!\left[\sqrt{D'(p_{\widehat\theta_n},p_0)}\right]
\le
\inf_{\theta\in\Theta}
\frac{L}{\sqrt{2m}}\,\sqrt{D(p_\theta,p_0)}
+
\left(
\sqrt{\frac{L}{2}}+\frac{L}{\sqrt{2m}}
\right)\mathbb E\!\left[\sqrt{D(\widehat p_n,p_0)}\right].
\]
\end{restatable}

The remainder of the flow in~\cite{cherief2022finite} can then be taken. 
In particular, we can study parameter estimation under the setting of non-iid samples from the target distribution. 
In order to do so, by Proposition~\ref{prop:universal-metric}, it suffices to bound $\sqrt{D(\widehat p_n,p_0)}$ in expectation or with high probability. 
We replicate here the results in expectation, but a similar extension can be given for the results which hold with high probability. 
In order to do so, we need to use an appropriate measure of dependence, as in~\cite{cherief2022finite}.
\begin{assumption}
\label{ass:kernelised_cov}
    Define a kernelised covariance quantity
    $$\rho_i = \mathbb{E} \big\langle k(X_i,\cdot) - \mu(p_0), \, k(X_0,\cdot) - \mu(p_0) \big\rangle_{\mathcal{H}},$$
    and assume that there exists a $\rho$ such that for any $n \geq 1$, the sum $\sum_{i=1}^n \rho_i \leq \rho$.
    Moreover, assume that $k \leq 1$.
\end{assumption} It is instructive to consider the case when $k(X_i,\cdot)$ is uncorrelated with $k(X_0, \cdot)$ (for example, if $X_i$ is independent of $X_0$) which implies that $\rho _i = \rho = 0$.
\begin{restatable}{proposition}{universalestimation}
Consider the setting of Theorem~\ref{thm:radial-bounds} and consider the estimator~\eqref{eq:estimator} with $D' = d_\Phi$. 
    Under Assumption~\ref{ass:kernelised_cov}, for any model, we may bound the expected square root divergence between the estimator and the target density,
    $$ \mathbb E\!\left[\sqrt{d_{\Phi}(p_{\widehat\theta_n},p_0)}\right]
\le
\inf_{\theta\in\Theta}
\frac{L_\phi(R)}{\sqrt{2m_\phi(R)}}\,\sqrt{\mathrm{MMD}_k(p_\theta,p_0)}
+
\sqrt\frac{1+\rho}{n} \left(
\sqrt{\frac{L_\phi(R)}{2}}+\frac{L_\phi(R)}{\sqrt{2m_\phi(R)}}
\right).$$
\end{restatable}
We make two remarks as also in~\cite{cherief2022finite}. 
First, the estimator is valid universally for any statistical model.
Second, the estimator is robust to contamination:
when $p_0=(1-\epsilon)p_{\theta_0} + \epsilon Q$ for some small $0<\epsilon<1$, with $p_{\theta_0}$ some well-specified mixture component and $Q$ some contaminating mixture component, the infimum scales only linearly with $\epsilon$. 
\section{Discussion and conclusion}

\subsection{Related works}
\paragraph{Pointwise Bregman divergence.} 
\cite{jones2002general} consider the problem of reconstructing a function from finitely many moment constraints. 
They do this by minimising some pointwise divergence of the form $\int d(f(x), g(x)) \, dx$ to a prior function $g$ subject to the moment constraints on the model $f$, where $d$ is some function. 
They restrict their divergences locally to  be evaluated on functions $f,g$ which are admissible (bounded below from zero, bounded above and integrable). 
They show that a Pythagorean identity is satisfied at the minimiser if and only if the divergence is a pointwise Bregman divergence. 
This pointwise Bregman divergence possesses many of the properties of finite dimensional Bregman divergences: nonnegativity and identity of indiscernibles, the squared difference is the only symmetric divergence, convexity in the first argument, duality and a Fenchel-Young inequality. 
They also show that the mean minimises the expected divergence with respect to the second argument.
The problem of minimising a similar pointwise divergence is considered in~\cite{csiszar1995generalized}.

\paragraph{Functional Bregman divergences.}
Surprisingly, Bregman himself noted the squared loss as a functional Bregman divergence (FBD) in a Hilbert space in his original 1967 paper~\citep{Bregman-1967}, but no properties of this divergence, nor general functional divergences, are elaborated upon. 
An FBD is studied in~\cite{bauschke2003bregman} on a general Banach space. 
They show a three point identity, four point identity, and consider Legendre duality. 
FBDs are built in~\cite{J-FunctionalBD-2008} on a Banach space $L^r(\nu)$ using a twice Fr\'echet differentiable, strictly convex and real valued functional on $L^r(\nu)$. 
Duality is considered via the dual space $L^s(\nu)$ where $\frac{1}{r} + \frac{1}{s} = 1$. 
Properties derived include nonnegativity and identity of indiscernibles, convexity in the first argument, reverse divergence as a dual divergence, the mean minimises the expected divergence in the second argument, and a three point identity. 
Similarly, FBDs on a convex subset of the set of probability measures are built in~\cite{ovcharov2018proper} and functional (scaled) Bregman divergences on the set of all probability measures appear in~\cite{stummer2012}.
Other functional Bregman divergences \citep{harandi2014bregman} focus on extending specific divergences for finite-dimensional positive semidefinite matrices to operators in a Hilbert space, with applications in computer vision.

\paragraph{Canonical divergence.}
Finite-dimensional Bregman divergences are the canonical divergences~\cite{amari2016information} on dually flat spaces which are Hessian manifolds~\cite{shima2013geometry} with global charts.
A particularly important FBD is built on statistical manifold through an Orlicz space in~\cite{pistonesempi1995}. 
This construct mirrors the well-known finite dimensional exponential family model; the primal generator is the log partition function and it recovers the KL divergence. 
The challenge in extending the Orlicz space construct in~\cite{pistonesempi1995} to general Bregman generators is that the probability density is not affine in the exponential coordinate, and hence the composition of the convex generator with the probability density in general fails to be convex in the exponential coordinate. 


\paragraph{Applications.}
In Bayesian quadratic discriminant analysis, functional BD has been used as a loss on class-conditional densities, so that the posterior-mean density is optimal in expected functional-Bregman risk and the resulting distribution-based classifier minimises expected misclassification cost \citep{srivastava2007bayesian}. This allows for naturally regularised Bayesian QDA for high dimensional or small data regimes. 
In time-series classification, functional BD has been used to compare kernel ridge regression representations of entire trajectories, yielding the squared RKHS distance between posterior-mean curves as a functional-Bregman divergence \cite{lu2008reproducing}. 
In this context, the reproducing property yields a closed-form expression for the functional BD. 
Several other papers use the label ``functional Bregman divergence'' we would label a pointwise Bregman divergence (as in~\citet{jones2002general}, who also consider sparse reconstruction as an application) between posterior densities, mainly for Bayesian model diagnostics and influence analysis \cite{goh2014bayesian,danilevicz2022bayesian}. 
Other related distributed-estimation work on densities in sensor networks is better described as density-space mirror descent with Bregman-type regularisation rather than as Fr\'echet-type FBD \cite{paritosh2025distributed}.
Matrix Bregman divergences have been applied to problems such as low-rank kernel learning~\citep{kulis2009low} and portfolio selection~\citep{nock2012mining}. 


\subsection{Conclusion}
While other FBDs have been introduced before, none have made use of the machinery of the reproducing kernel Hilbert space, which is particularly well-suited to machine learning. 
We described some general properties of an FBD built on a Hilbert space, making use of the self-duality of a Hilbert space and the Riesz representation of the gradient. 
Further specialising, when the generator is composed with an embedding into a different RKHS, the resulting k-FBD is easy to estimate. 
The derived k-FBD generalises the squared MMD, and notable other cases are deformed squared MMD and a divergence induced by a kernelised negative entropy generator. 
We discussed applications, including the problem of universal and robust estimation, where we used a sandwich bound with the MMD to show that the deformed MMD is a universal and robust estimator. 

Several fundamental properties of our FBD remain open, including whether the mean-as-minimiser property implies that the divergence is an FBD. 
It would be interesting to further develop the tool of k-FBD for hypothesis testing, generative modelling, clustering, and other applications in machine learning and statistics.

\section*{Acknowledgments}
Russell Tsuchida is supported through an Australian Research Council DECRA (DE260100189). 
He would also like to thank Sony Computer Science Laboratories, Inc. Tokyo for allowing him to visit. 

\bibliographystyle{plainnat}
\bibliography{refs}

@book{amari2016information,
  title={Information geometry and its applications},
  author={Amari, Shun-ichi},
  year={2016},
  publisher={Springer}
}

@inproceedings{shima2013geometry,
  title={{Geometry of Hessian structures}},
  author={Shima, Hirohiko},
  booktitle={International Conference on Geometric Science of Information},
  pages={37--55},
  year={2013},
  organization={Springer}
}

@book{barndorff2014information,
  title={Information and exponential families},
  author={Barndorff-Nielsen, Ole},
  year={2014},
  publisher={John Wiley \& Sons}
}

@inproceedings{frongillo2014convex,
  title={{Convex foundations for generalized MaxEnt models}},
  author={Frongillo, Rafael and Reid, Mark D},
  booktitle={AIP Conference Proceedings},
  volume={1636},
  number={1},
  pages={11--16},
  year={2014},
  organization={American Institute of Physics}
}

@article{J-FunctionalBD-2008,
  title={{Functional Bregman divergence and Bayesian estimation of distributions}},
  author={Frigyik, B{\'e}la A and Srivastava, Santosh and Gupta, Maya R},
  journal={IEEE Transactions on Information Theory},
  volume={54},
  number={11},
  pages={5130--5139},
  year={2008},
  publisher={IEEE}
}

@article{Bregman-1967,
  title={The relaxation method of finding the common point of convex sets and its application to the solution of problems in convex programming},
  author={Bregman, Lev M.},
  journal={USSR Computational Mathematics and Mathematical Physics},
  volume={7},
  number={3},
  pages={200--217},
  year={1967},
  publisher={Elsevier},
  NFnote={$D$-projection, relaxation sequence, relaxation control (e.g., farthest projection). BD given in 1.4 with example of squared quadratic distance and ext KLD.
  Hilbert distance with weak topology.}
}

@book{bauschke2020,
  author={Bauschke, Heinz H and Combettes, Patrick L},
  title={{Convex Analysis and Monotone Operator Theory in Hilbert Spaces}},
  year={2020},
  publisher={Springer}
}

@inproceedings{acharyya2013bregman,
  title={Bregman divergences and triangle inequality},
  author={Acharyya, Sreangsu and Banerjee, Arindam and Boley, Daniel},
  booktitle={Proceedings of the 2013 SIAM International Conference on Data Mining},
  pages={476--484},
  year={2013},
  organization={SIAM}
}

@article{cherief2022finite,
  title={{Finite sample properties of parametric MMD estimation: robustness to misspecification and dependence}},
  author={Ch{\'e}rief-Abdellatif, Badr-Eddine and Alquier, Pierre},
  journal={Bernoulli},
  volume={28},
  number={1},
  pages={181--213},
  year={2022},
  publisher={Bernoulli Society for Mathematical Statistics and Probability}
}

@article{pistonesempi1995,
author = {Giovanni Pistone and Carlo Sempi},
title = {{An Infinite-Dimensional Geometric Structure on the Space of all the Probability Measures Equivalent to a Given One}},
volume = {23},
journal = {The Annals of Statistics},
number = {5},
publisher = {Institute of Mathematical Statistics},
pages = {1543 -- 1561},
year = {1995}
}

@article{ovcharov2018proper,
  title={{Proper scoring rules and Bregman divergence}},
  author={Ovcharov, Evgeni Y},
  journal={Bernoulli},
  pages={53--79},
  year={2018}
}

@ARTICLE{stummer2012,
  author={Stummer, Wolfgang and Vajda, Igor},
  journal={IEEE Transactions on Information Theory}, 
  title={{On Bregman Distances and Divergences of Probability Measures}}, 
  year={2012},
  volume={58},
  number={3},
  pages={1277-1288}}

@article{jones2002general,
  title={General entropy criteria for inverse problems, with applications to data compression, pattern classification, and cluster analysis},
  author={Jones, Lee K and Byrne, Charles L},
  journal={IEEE Transactions on Information Theory},
  volume={36},
  number={1},
  pages={23--30},
  year={2002},
  publisher={IEEE}
}

@inproceedings{csiszar1995generalized,
  title={Generalized projections for non-negative functions},
  author={Csisz{\'a}r, Imre},
  booktitle={Proceedings of 1995 IEEE International Symposium on Information Theory},
  pages={6},
  year={1995},
  organization={IEEE}
}

@article{bauschke2003bregman,
  title={Bregman monotone optimization algorithms},
  author={Bauschke, Heinz H and Borwein, Jonathan M and Combettes, Patrick L},
  journal={SIAM Journal on control and optimization},
  volume={42},
  number={2},
  pages={596--636},
  year={2003},
  publisher={SIAM}
}

@article{gretton2006kernel,
  title={A kernel method for the two-sample-problem},
  author={Gretton, Arthur and Borgwardt, Karsten and Rasch, Malte and Sch{\"o}lkopf, Bernhard and Smola, Alex},
  journal={Advances in neural information processing systems},
  volume={19},
  year={2006}
}

@article{
gretton2012kernel,
  title={A kernel two-sample test},
  author={Gretton, Arthur and Borgwardt, Karsten M and Rasch, Malte J and Sch{\"o}lkopf, Bernhard and Smola, Alexander},
  journal={The journal of machine learning research},
  volume={13},
  number={1},
  pages={723--773},
  year={2012}
}

@article{sriperumbudur2010hilbert,
  title={Hilbert space embeddings and metrics on probability measures},
  author={Sriperumbudur, Bharath K and Gretton, Arthur and Fukumizu, Kenji and Sch{\"o}lkopf, Bernhard and Lanckriet, Gert RG},
  journal={The Journal of Machine Learning Research},
  volume={11},
  pages={1517--1561},
  year={2010}
}

@inproceedings{lu2008reproducing,
  title={{A reproducing kernel Hilbert space framework for pairwise time series distances}},
  author={Lu, Zhengdong and Leen, Todd K and Huang, Yonghong and Erdogmus, Deniz},
  booktitle={Proceedings of the 25th international conference on Machine learning},
  pages={624--631},
  year={2008},
  NFnote={mention functional BD and GP},
  russnote={functional BD for timeseries. Convert timeseries data to GP regression mean/krr, then compute squared distance in RKHS. Reduction via reproducing property.}
}

@article{paritosh2025distributed,
  title={{Distributed Bayesian Estimation in Sensor Networks: Consensus on Marginal Densities}},
  author={Paritosh, Parth and Atanasov, Nikolay and Mart{\'\i}nez, Sonia},
  journal={IEEE Transactions on Network Science and Engineering},
  year={2025},
  publisher={IEEE},
  NFnote={functional BD},
  russnote={Considers POINTWISE Bregman divergence (in particular, the special case of KL divergence) trying to infer a distribution over unobserved latent variables in distributed sensor networks (e.g. true temperature distribution given a collection of measurements from different sensors)}
}

@article{goh2014bayesian,
  title={{Bayesian model diagnostics using functional Bregman divergence}},
  author={Goh, Gyuhyeong and Dey, Dipak K},
  journal={Journal of Multivariate Analysis},
  volume={124},
  pages={371--383},
  year={2014},
  publisher={Elsevier},
  russnote={POINTWISE Bregman divergence for measuring sensitivity of posterior to changes in prior and likelihood, for iid data samples.}
}

@article{danilevicz2022bayesian,
  title={Bayesian influence diagnostics using normalized functional Bregman divergence},
  author={Danilevicz, Ian M and Ehlers, Ricardo S},
  journal={Communications in Statistics-Theory and Methods},
  volume={51},
  number={6},
  pages={1637--1652},
  year={2022},
  publisher={Taylor \& Francis},
  NFnote={pointwise BD},
  russnote={POINTWISE BD extending Goh et al. to non-iid data. Also normalised to give easier diagnostics.}
}

@article{srivastava2007bayesian,
  title={Bayesian quadratic discriminant analysis},
  author={Srivastava, Santosh and Gupta, Maya R and Frigyik, B{\'e}la A},
  journal={Journal of Machine Learning Research},
  volume={8},
  number={6},
  year={2007},
  russnotes={mean as minimiser property used. The estimated class density is chosen to minimize posterior expected Bregman divergence over distributions. the resulting distribution-based classifier is shown to minimize expected misclassification cost. Motivated by ill-posed covariance estimation in high-dimensional, small-sample classification. This extends QDA to a Bayesian setting, allowing natural probabilistic regression in ill-posed settings (e.g. high dimensions and not many samples).}
}

@article{banerjee2005,
  author  = {Arindam Banerjee and Srujana Merugu and Inderjit S. Dhillon and Joydeep Ghosh},
  title   = {{Clustering with Bregman Divergences}},
  journal = {Journal of Machine Learning Research},
  year    = {2005},
  volume  = {6},
  number  = {58},
  pages   = {1705--1749}
}

@article{blondel2020learning,
  title={{Learning with Fenchel-Young losses}},
  author={Blondel, Mathieu and Martins, Andr{\'e} FT and Niculae, Vlad},
  journal={Journal of Machine Learning Research},
  volume={21},
  number={35},
  pages={1--69},
  year={2020}
}

@article{nielsen2011burbea,
  title={{The Burbea-Rao and Bhattacharyya centroids}},
  author={Nielsen, Frank and Boltz, Sylvain},
  journal={IEEE Transactions on Information Theory},
  volume={57},
  number={8},
  pages={5455--5466},
  year={2011},
  publisher={IEEE}
}

@article{banerjee2005optimality,
  title={{On the optimality of conditional expectation as a Bregman predictor}},
  author={Banerjee, Arindam and Guo, Xin and Wang, Hui},
  journal={IEEE Transactions on Information Theory},
  volume={51},
  number={7},
  pages={2664--2669},
  year={2005},
  publisher={IEEE}
}

@article{schoenberg1938metric,
  title={Metric spaces and positive definite functions},
  author={Schoenberg, Isaac J},
  journal={Transactions of the American Mathematical Society},
  volume={44},
  number={3},
  pages={522--536},
  year={1938}
}

@inproceedings{harandi2014bregman,
  title={Bregman divergences for infinite dimensional covariance matrices},
  author={Harandi, Mehrtash and Salzmann, Mathieu and Porikli, Fatih},
  booktitle={Proceedings of the IEEE Conference on Computer Vision and Pattern Recognition},
  pages={1003--1010},
  year={2014}
}

@article{nielsen2020elementary,
  title={An elementary introduction to information geometry},
  author={Nielsen, Frank},
  journal={Entropy},
  volume={22},
  number={10},
  pages={1100},
  year={2020}
}

@article{collins2002logistic,
  title={{Logistic regression, AdaBoost and Bregman distances}},
  author={Collins, Michael and Schapire, Robert E and Singer, Yoram},
  journal={Machine Learning},
  volume={48},
  number={1},
  pages={253--285},
  year={2002},
  publisher={Springer}
}

@article{collins2001generalization,
  title={A generalization of principal components analysis to the exponential family},
  author={Collins, Michael and Dasgupta, Sanjoy and Schapire, Robert E},
  journal={Advances in neural information processing systems},
  volume={14},
  year={2001}
}

@article{wild2022generalized,
  title={{Generalized variational inference in function spaces: Gaussian measures meet Bayesian deep learning}},
  author={Wild, Veit David and Hu, Robert and Sejdinovic, Dino},
  journal={Advances in Neural Information Processing Systems},
  volume={35},
  pages={3716--3730},
  year={2022}
}

@article{beck2003mirror,
  title={Mirror descent and nonlinear projected subgradient methods for convex optimization},
  author={Beck, Amir and Teboulle, Marc},
  journal={Operations Research Letters},
  volume={31},
  number={3},
  pages={167--175},
  year={2003},
  publisher={Elsevier}
}

@book{ScholkopfSmola2002,
  author    = {Bernhard Sch{\"o}lkopf and Alexander J. Smola},
  title     = {Learning with Kernels: Support Vector Machines, Regularization, Optimization, and Beyond},
  publisher = {MIT Press},
  address   = {Cambridge, MA},
  year      = {2002}
}

@article{knoblauch2019generalized,
  title={Generalized variational inference: Three arguments for deriving new posteriors},
  author={Knoblauch, Jeremias and Jewson, Jack and Damoulas, Theodoros},
  journal={arXiv preprint arXiv:1904.02063},
  year={2019}
}

@book{RasmussenWilliams2006,
  author    = {Carl Edward Rasmussen and Christopher K. I. Williams},
  title     = {Gaussian Processes for Machine Learning},
  publisher = {MIT Press},
  address   = {Cambridge, MA},
  year      = {2006}
}

@article{arbel2019maximum,
  title={Maximum mean discrepancy gradient flow},
  author={Arbel, Michael and Korba, Anna and Salim, Adil and Gretton, Arthur},
  journal={Advances in neural information processing systems},
  volume={32},
  year={2019}
}

@inproceedings{songscore,
  title={Score-Based Generative Modeling through Stochastic Differential Equations},
  author={Song, Yang and Sohl-Dickstein, Jascha and Kingma, Diederik P and Kumar, Abhishek and Ermon, Stefano and Poole, Ben},
  booktitle={International Conference on Learning Representations},
  year={2021}
}

@inproceedings{elkan2003using,
  title={Using the triangle inequality to accelerate k-means},
  author={Elkan, Charles},
  booktitle={Proceedings of the 20th international conference on Machine Learning},
  pages={147--153},
  year={2003}
}

@inproceedings{li2015generative,
  title={Generative moment matching networks},
  author={Li, Yujia and Swersky, Kevin and Zemel, Rich},
  booktitle={International conference on machine learning},
  pages={1718--1727},
  year={2015}
}

@article{li2017mmd,
  title={{MMD GAN: Towards deeper understanding of moment matching network}},
  author={Li, Chun-Liang and Chang, Wei-Cheng and Cheng, Yu and Yang, Yiming and P{\'o}czos, Barnab{\'a}s},
  journal={Advances in neural information processing systems},
  year={2017}
}

@article{LegendreType-1967,
   title={{Conjugates and Legendre transforms of convex functions}},
   author={Rockafellar, R. T.},
   journal={Canadian Journal of Mathematics},
   volume={19},
   pages={200--205},
   year={1967},
   publisher={Cambridge University Press},
     NFnote={introduces Legendre-type function, counter example for
non-convex domain of a Legendre transform of a convex function}
}

@article{azoury2001relative,
   title={Relative loss bounds for on-line density estimation with the
exponential family of distributions},
   author={Azoury, Katy S and Warmuth, Manfred K},
   journal={Machine learning},
   volume={43},
   pages={211--246},
   year={2001},
   publisher={Springer}
}

@article{kulis2009low,
   title={{Low-rank kernel learning with Bregman matrix divergences}},
   author={Kulis, Brian and Sustik, M{\'a}ty{\'a}s A and Dhillon,
Inderjit S},
   journal={Journal of Machine Learning Research},
   volume={10},
   pages={341--376},
   year={2009}
}

@incollection{nock2012mining,
   title={{Mining matrix data with Bregman matrix divergences for
portfolio selection}},
   author={Nock, Richard and Magdalou, Brice and Briys, Eric and
Nielsen, Frank},
   booktitle={Matrix Information Geometry},
   pages={373--402},
   year={2012},
   publisher={Springer}
}

@article{benamou2015iterative,
   title={{Iterative Bregman Projections for Regularized Transportation
Problems}},
   author={Benamou, Jean-David and Carlier, Guillaume and Cuturi, Marco
and Nenna, Luca and Peyre, Gabriel},
   journal={SIAM Journal on Scientific Computing},
   volume={37},
   number={2},
   pages={A1111--A1138},
   year={2015},
   publisher={Society for Industrial and Applied Mathematics},
     NFnote={}
}

@article{raskutti2015information,
   title={The information geometry of mirror descent},
   author={Raskutti, Garvesh and Mukherjee, Sayan},
   journal={IEEE Transactions on Information Theory},
   volume={61},
   number={3},
   pages={1451--1457},
   year={2015},
   publisher={IEEE}
}

@article{ProximalGenBD-1997,
   title={{Proximal minimization methods with generalized Bregman
functions}},
   author={Kiwiel, Krzysztof C},
   journal={SIAM journal on control and optimization},
   volume={35},
   number={4},
   pages={1142--1168},
   year={1997},
   publisher={SIAM}
}
\clearpage

\appendix
\section{Fr\'echet derivatives and gradients}
\label{app:frechet}
In order to construct Bregman divergences on Hilbert spaces, we require an appropriate notion of differentiability of real-valued functionals. For our purposes, Fr\'echet derivatives suffice. We expose some basics on Fr\'echet derivatives here, and refer the reader to~\cite[Definition 2.56]{bauschke2020} for extended discussion (noting also that Fr\'echet differentiability is a stronger notion than Gateaux differentibility). 
We first recall a standard definition.

\begin{definition}
\label{def:frechet}
Let $\mathcal E$ be a real Hilbert space, let $U\subseteq \mathcal E$ be open, and let $T:U\to \mathbb R$. We say that $T$ is \emph{Fr\'echet differentiable} at $u\in U$ if there exists a bounded linear functional $DT(u):\mathcal E\to\mathbb R$ such that
\[
\lim_{\|h\|_{\mathcal E}\to 0}
\frac{|T(u+h)-T(u)-DT(u)[h]|}{\|h\|_{\mathcal E}}=0.
\]
Since $\mathcal E$ is Hilbert, the Riesz representation theorem yields a unique element $\nabla T(u)\in\mathcal E$ such that
\[
DT(u)[h]=\langle h,\nabla T(u)\rangle_{\mathcal E}
\qquad \text{for all } h\in\mathcal E.
\]
This element is called the \emph{Fr\'echet gradient} of $T$ at $u$.
\end{definition}

\subsection{Examples in Hilbert space}
\label{app:frechet_hilbert}
One way in which Definition~\ref{def:frechet} can be used for calculation is to expand $T$ up to second order perturbations, as the following example shows. 
\begin{example}[Squared norm]
\label{example:frechet_squared_norm}
Consider the functional $T:\mathcal H\to \mathbb R$ given by $T(u)=\|u\|_{\mathcal H}^2$. Expanding $T(u+\varepsilon h)$ gives $T(u+\varepsilon h)=\|u\|_{\mathcal H}^2+2\varepsilon\langle u,h\rangle_{\mathcal H}+O(\varepsilon^2)$, and therefore
\[
DT(u)[h]=2\langle u,h\rangle_{\mathcal H},
\qquad
\nabla T(u)=2u.
\]
\end{example}
The chain rule takes the form $D(R\circ T)(u)[h]=DR(T(u))[DT(u)[h]]$ whenever the composition is well-defined and the maps are Fr\'echet differentiable. 
The following example uses the chain rule.
\begin{example}[Radial functional]
\label{example:frechet_radial}
Let $\phi:[0,\infty)\to \mathbb R$ be differentiable, and define $T:\mathcal H\setminus\{0\}\to\mathbb R$ by $T(u)=\phi(\|u\|_{\mathcal H})$. By the chain rule, $DT(u)[h]=\phi'(\|u\|_{\mathcal H})\,\langle u,h\rangle_{\mathcal H}/\|u\|_{\mathcal H}$, so
\[
\nabla T(u)=\frac{\phi'(\|u\|_{\mathcal H})}{\|u\|_{\mathcal H}}\,u.
\]
\end{example}
Fr\'echet gradients extend to higher orders, a helpful example being the Hessian.
\begin{example}[Hessian of a radial functional]
\label{example:frechet_hessian}
Let $\phi:[0,\infty)\to\mathbb R$ be twice differentiable on $(0,\infty)$, and let $T(u)=\phi(\|u\|_{\mathcal H})$ for $u\in\mathcal H\setminus\{0\}$. Writing $a(r)=\phi'(r)/r$, we have $\nabla T(u)=a(\|u\|_{\mathcal H})u$. Using $\|u+\varepsilon h\|_{\mathcal H}=\|u\|_{\mathcal H}+\varepsilon\langle u/\|u\|_{\mathcal H},h\rangle_{\mathcal H}+O(\varepsilon^2)$, one obtains
\[
\nabla^2T(u)[h]
=
\frac{\phi'(\|u\|_{\mathcal H})}{\|u\|_{\mathcal H}}\,h
+
\frac{\phi''(\|u\|_{\mathcal H})\|u\|_{\mathcal H}-\phi'(\|u\|_{\mathcal H})}{\|u\|_{\mathcal H}^3}
\langle u,h\rangle_{\mathcal H}\,u.
\]
Equivalently, if $P_u^\parallel=(u\otimes u)/\|u\|_{\mathcal H}^2$ and $P_u^\perp=I_{\mathcal H}-P_u^\parallel$, then
\[
\nabla^2T(u)
=
\frac{\phi'(\|u\|_{\mathcal H})}{\|u\|_{\mathcal H}}\,P_u^\perp
+
\phi''(\|u\|_{\mathcal H})\,P_u^\parallel.
\]
Thus the Hessian has the spectral decomposition $\nabla^2T(u)=\lambda_u^\perp P_u^\perp+\lambda_u^\parallel P_u^\parallel$, where $\lambda_u^\perp=\phi'(\|u\|_{\mathcal H})/\|u\|_{\mathcal H}$ and $\lambda_u^\parallel=\phi''(\|u\|_{\mathcal H})$.
\end{example}

\subsection{Examples with kernel mean embedding}
\label{app:frechet_kme}
We take an operator $A:\mathcal{F} \to \mathbb{R}$ corresponding with the norm of the kernel mean embedding $\mu(g)$,
\begin{align*}
    \mu(g) = \int k(\cdot, y) g(y) \, \nu(dy), \qquad A(g) = \Vert \mu(g) \Vert_{\mathcal{H}}.
\end{align*}
This implies $\Vert A(g) \Vert_{\mathcal{H}} = \int \int k(x,y) g(x) g(y) \nu (dx) \nu(dy)$. 
Note that $\mu$ is a linear operator with an adjoint $\mu^\ast:\mathcal{H}\to\mathcal{F}$, which is defined to be an operator satisfying~\citep[Equation (2.21)]{bauschke2020} $$\langle \mu(g), h \rangle_{\mathcal{H}} = \langle g, \mu^\ast(h) \rangle_{\mathcal{F}}.$$

We consider functionals formed via composition with the norm of the kernel mean embedding of the form $\Phi(g)=\phi\big(A(g)\big)$, examining the gradients using the chain rule and the examples in  Appendix~\ref{app:frechet_hilbert}. 
In particular,
\begin{example}[Radial functional]
\label{example:frechet_radial_composed}
Let $\phi:[0,\infty)\to \mathbb R$ be differentiable, and define $\Phi:\mathcal{F}\setminus\{0\}\to\mathbb R$ by $\Phi(g)=\phi\big(A(g)\big)$. By the chain rule, 
\begin{align*}
D\Phi(g)[h]&= D(\phi\circ A)(g)[h]\\
&=D\phi(A(g))[DA(g)[h]]\\
&=\phi'\big(A (g) \big) \Big[ \frac{1}{2 A(g) } \mu^\ast \mu (g) \big[ h\big]\Big] \\
&=\phi'\big(A (g) \big)  \frac{1}{2 A(g) } \langle \mu^\ast \mu (g), h \rangle_{\mathcal{F}} \\
&= \phi'\big(A (g) \big)  \frac{1}{2 A(g) } \langle \mu (g), \mu(h) \rangle_{\mathcal{H}},
\end{align*}
where the second last line follows by the Riesz representation theorem. 
Hence it is easiest to interpret the Fr\'echet gradient as an operator acting on the kernel mean embedding in $\mathcal{H}$, rather than on the original space $\mathcal{F}$.
\end{example}


\section{Special cases of k-FBD}
\label{app:special_cases}
Let $G:\mathcal{H}\to\mathcal{H}$ and consider generators $F:\mathcal{H}\to\mathbb{R}$ of the form
\begin{align*}
    F\big(\mu(p)\big) &= \langle \mu(p), G\big(\mu(p)\big) \rangle_{\mathcal{H}} = \mathbb{E}_{x\sim p} \big[ G\big(\mu(p) \big) (x) \big] ,
\end{align*}
where the second equality is by the reproducing property of the kernel mean embedding. 
Formally, the first and second Fr\'echet gradients of $F$ are
\begin{align*}
    \nabla F(\mu) &= G(\mu) +  \nabla G(\mu)^\ast \mu  \\
    \nabla^2 F(\mu)[h,g]  &= \langle h, \nabla G(\mu)g \rangle_{\mathcal{H}} + \langle g, \nabla G(\mu) h \rangle_{\mathcal{H}} + \langle \mu,  \nabla^2 G(\mu)[h,g] \rangle_{\mathcal{H}}.
\end{align*}
Hence the k-FBD is, using the reproducing property of the kernel mean embedding, 
\begin{align*}
    d_{F}\big( \mu(p), \mu(q) \big)
    &= F\big( \mu(p) \big) - F\big( \mu(q) \big)
       - \big\langle \nabla F\big( \mu(q) \big), \mu(p) - \mu(q) \big\rangle_{\mathcal{H}} \\
    &= \langle \mu(p), G(\mu(p)) \rangle_{\mathcal H}
       - \langle \mu(q), G(\mu(q)) \rangle_{\mathcal H}
       - \langle G(\mu(q)), \mu(p) - \mu(q) \rangle_{\mathcal H} 
       - \big\langle \mu(q), \nabla G(\mu(q))\big(\mu(p) - \mu(q)\big) \big\rangle_{\mathcal H} \\
    &= \mathbb{E}_{x\sim p}\!\left[ G(\mu(p))(x) \right]
       - \mathbb{E}_{x\sim p}\!\left[ G(\mu(q))(x) \right]
       - \mathbb{E}_{y\sim q}\!\left[ \nabla G(\mu(q))\big(\mu(p)-\mu(q)\big)(y) \right].
\end{align*}
Formally, some interesting special cases are
\begin{itemize}
    \item When $G$ is the identity, we recover the squared MMD, and $\nabla G(\mu) = \text{Id}$ and $\nabla^2 G (\mu) = 0$.
    \item When $G(\mu) = \sigma(\Vert \mu \Vert_{\mathcal{H}}) \mu,$ we recover the deformed squared MMD, as studied in \S~\ref{sec:deformed}, up to a reparameterisation $\sigma(\Vert \mu \Vert_{\mathcal{H}})\Vert \mu \Vert_{\mathcal{H}}^2 = \phi\big(\Vert \mu \Vert_{\mathcal{H}}\big).$
    \item When $G(\mu) = \int_{\mathcal{X}} \sigma\big( \mu(x) \big) k(x, \cdot) \, d\nu(x)$ for some $\nu$-integrable $\sigma:\mathbb{R}\to\mathbb{R}$ and $k$ is the bounded reproducing kernel of $\mathcal{H}$, we find, using the reproducing property,
    $$F(\mu) = \int \sigma\big( \mu(x) \big) \langle \mu, k(x,\cdot)  \rangle_{\mathcal{H}} \, d\nu(x) = \int_{\mathcal{X}} \sigma\big( \mu(x) \big) \mu(x) \, d\nu(x).$$ Hence we may recover a pointwise generator in $L^2(\nu)$. 
    A prominent example is a kernelised negative entropy generator, obtained with $\sigma = \log$.
\end{itemize}

\section{Bochner means in a Hilbert space}
\label{app:bochner}
This appendix very briefly touches on the notion of expectation for random elements taking values in a Hilbert space. 
For a more detailed background on Banach-space-valued random variables and the Bochner integral.

\paragraph{Probability space and approximation by simple functions.} 
Let $(\Omega,\Sigma,\mathbb{P})$ be a probability space, and let $\mathcal{F}$ be a real separable Hilbert space with inner product $\langle \cdot,\cdot\rangle_{\mathcal{F}}$ and norm $\|\cdot\|_{\mathcal{F}}$. Since every Hilbert space is, in particular, a Banach space, the standard theory of Bochner integration applies to $\mathcal{F}$-valued random elements.
A random element in $\mathcal{F}$ is a measurable map $X \colon \Omega \to \mathcal{F}$, where $\mathcal{F}$ is equipped with its Borel $\sigma$-algebra. 
A simple $\mathcal{F}$-valued random element $S$ is one which can be written as a finite linear combination of random evaluations of indicator functions multiplied by some deterministic elements of $\mathcal{F}$, that is, $S(\omega)=\sum_{i=1}^n f_i \mathbf{1}_{A_i}(\omega)$, where $f_1,\dots,f_n \in \mathcal{F}$ and $A_1,\dots,A_n \in \Sigma$ are measurable. Its expectation is defined by $\mathbb{E}[S]=\sum_{i=1}^n \mathbb{P}(A_i)f_i$.
More generally, if $X$ is measurable and satisfies $\mathbb{E}\|X\|_{\mathcal{F}}<\infty$, then $X$ is Bochner integrable, and its expectation $\mathbb{E}[X]\in\mathcal{F}$ is defined as its Bochner integral. Equivalently, one may define $\mathbb{E}[X]$ by approximating $X$ in $L^1(\Omega;\mathcal{F})$ by simple random elements and passing to the limit. 

\paragraph{Inner products of Bochner means in Hilbert space.} 
In our current setting of a Hilbert space, the Bochner mean is characterised by testing against fixed elements of the Hilbert space through inner product evaluations. 
Specifically, if $X$ is Bochner integrable, then for every $h\in\mathcal{F}$,
\begin{equation}
\label{eq:bochner-pairing}
\bigl\langle \mathbb{E}[X],h \bigr\rangle_{\mathcal{F}}
=
\mathbb{E}\bigl[\langle X,h\rangle_{\mathcal{F}}\bigr].
\end{equation}
(The scalar random variable $\langle X,h\rangle_{\mathcal{F}}$ is integrable by Cauchy Schwarz, since $|\langle X,h\rangle_{\mathcal{F}}|\leq \|X\|_{\mathcal{F}}\|h\|_{\mathcal{F}}$ and $\mathbb{E}\|X\|_{\mathcal{F}}<\infty$).
Equation \eqref{eq:bochner-pairing} is the Hilbert space analogue of the familiar scalar identity defining ordinary expectation. It shows that $\mathbb{E}[X]$ is the unique element of $\mathcal{F}$ whose inner product against any test direction $h$ agrees with the expectation of the corresponding scalar projection of $X$.

As in the scalar case, the Bochner integral is linear: if $X$ and $Y$ are Bochner integrable and $a,b\in\mathbb{R}$, then $\mathbb{E}[aX+bY]=a\,\mathbb{E}[X]+b\,\mathbb{E}[Y]$. More generally, if $T\colon\mathcal{F}\to\mathcal{F}$ is bounded and linear and $X$ is Bochner integrable, then $T(\mathbb{E}[X])=\mathbb{E}[T(X)]$. Taking $T(\cdot)=\langle \cdot,h\rangle_{\mathcal{F}}$ recovers \eqref{eq:bochner-pairing}.
The only identity actually needed is the special case of \eqref{eq:bochner-pairing} in which the first slot is deterministic: if $u\in\mathcal{F}$ is fixed and $X$ is Bochner integrable, then
\[
\mathbb{E}\bigl[\langle u,X\rangle_{\mathcal{F}}\bigr]
=
\bigl\langle u,\mathbb{E}[X]\bigr\rangle_{\mathcal{F}}.
\]
Hence, if $\bar f:=\mathbb{E}[X]$, then for any fixed $g\in\mathcal{F}$ one has $\mathbb{E}[\langle u,X-g\rangle_{\mathcal{F}}]=\langle u,\bar f-g\rangle_{\mathcal{F}}$. This is exactly the step used to replace a random inner-product term by a deterministic one in the bias--variance decomposition underlying the mean-as-minimiser result.

\section{Proofs}
\label{app:proofs}

\propbasic*
\begin{proof}
We prove each of the claims separately.
    \begin{enumerate}
        \item Follows the same reasoning as in~\cite{J-FunctionalBD-2008} for $L^r$ space, which also works for Hilbert spaces, and which we review here for completeness. Define $\widetilde{\Phi}:\mathbb{R}\to\mathbb{R}$ as $\widetilde{\Phi}(t)=\Phi\big(t f + (1-t) g \big)$, which is convex. 
        By the mean value theorem and the fact that $\widetilde{\Phi}'$ is increasing, $\widetilde{\Phi}(1) - \widetilde{\Phi}(0) \geq \widetilde{\Phi}'(0)$, i.e. $\Phi(t) - \Phi(g) \geq \langle \nabla \Phi(g), f - g \rangle_{\mathcal{F}}$, proving nonnegativity. It is clear that $d_{\Phi}[f,g]=0$ if $f=g$. For the reverse, zero FBD implies that $\widetilde{\Phi}(1) - \widetilde{\Phi}(0) = \widetilde{\Phi}'(0)$. Since $\widetilde{\Phi}$ is strictly convex, if $f\neq g$, then $\widetilde{\Phi}(t) < (1-t) \widetilde{\Phi}(0) + t \widetilde{\Phi}(1)= \widetilde{\Phi}(0) + t \widetilde{\Phi}'(0)$. But on the other hand, taking the tangent at $0$, convexity implies the opposite, $\widetilde{\Phi}(t) \geq \widetilde{\Phi}(0) + t \widetilde{\Phi}'(0)$. Hence it cannot be that $f \neq g$.
        \item With the second argument $q$ fixed, $d_{\Phi}(p,q)$ is the sum of strictly convex $\Phi$, a linear inner product and a constant. 
        Hence it is strictly convex.
        \item This is a direct calculation.
        \item This is also a direct calculation, as follows.
        \begin{align*}
            d_{\Phi}(f,g) - d_{\Phi}(f,h) &= \Phi(h) - \Phi(g) - \langle \nabla \Phi(g), f-g \rangle_{\mathcal{F}} + \langle \nabla \Phi(h), f-h \rangle_{\mathcal{F}} \\
            &= \Phi(h) - \Phi(g) - \langle \nabla \Phi(g), h-g + (f - h)\rangle_{\mathcal{F}} + \langle \nabla \Phi(h), f-h \rangle_{\mathcal{F}} \\
            &= d_{\Phi}(h,g) - \langle \nabla\Phi(g) - \nabla\Phi(h), f - h  \rangle_{\mathcal{F}}
        \end{align*}
    \end{enumerate}
\end{proof}

\quasimean*
\begin{proof}
Define $Z=\nabla\Phi(F)$. By Legendre duality for FBDs, for all $g\in \mathcal{U}$ and all $f\in \mathcal{U}$,
\[
d_\Phi(g,f) \;=\; d_{\Phi^\ast}\bigl(\nabla\Phi(f),\,\nabla\Phi(g)\bigr).
\]
Applying this with $f=F$ yields, almost surely,
\[
d_\Phi(g,F) \;=\; d_{\Phi^\ast}\bigl(Z,\,\nabla\Phi(g)\bigr),
\]
and therefore
\[
\mathbb E[d_\Phi(g,F)] \;=\; \mathbb E\Bigl[d_{\Phi^\ast}\bigl(Z,\,\nabla\Phi(g)\bigr)\Bigr].
\]
Now set $v=\nabla\Phi(g)$. Since $\nabla\Phi:\mathcal{U}\to V$ is a bijection, minimising over $g\in \mathcal{U}$ is equivalent to
minimising over $v\in V$,
\[
\argmin_{g\in \mathcal{U}}\; \mathbb E[d_\Phi(g,F)]
\;=\;
(\nabla\Phi)^{-1}\!\left(
\argmin_{v\in V}\; \mathbb E\bigl[d_{\Phi^\ast}(Z,v)\bigr]
\right).
\]
Hence 
the unique minimiser of $v\mapsto \mathbb E[d_{\Phi^\ast}(Z,v)]$ is $\bar z=\mathbb E[Z]$.
Hence the unique minimiser in primal coordinates is
\[
g^\star \;=\; (\nabla\Phi)^{-1}(\bar z)\;=\;\nabla\Phi^\ast(\bar z)
\;=\;\nabla\Phi^\ast\!\left(\mathbb E[\nabla\Phi(F)]\right).
\]
\end{proof}

\symmetric*
\begin{proof}
$(2)\Rightarrow(3)$ is immediate, since $\nabla\Phi(f)=Tf+b$, and therefore
$d_\Phi(f,g)=\frac12\langle f-g,T(f-g)\rangle_F$. $(3)\Rightarrow(1)$ is also immediate. 
It remains to prove $(1)\Rightarrow(2)$. Symmetry gives
\[
2(\Phi(f)-\Phi(g))=\langle \nabla\Phi(f)+\nabla\Phi(g),\,f-g\rangle_{\mathcal{F}}
\qquad (f,g\in \mathcal{A}).
\]
Differentiating with respect to $f$, for all $h\in \mathcal{F}$ we find 
\[
\langle \nabla\Phi(f)-\nabla\Phi(g),h\rangle_{\mathcal{F}}
=
\langle D^2\Phi(f)h,\,f-g\rangle_{\mathcal{F}}
\qquad (f,g\in \mathcal{A},\ h\in \mathcal{F}),
\]
hence
\[
\nabla\Phi(f)-\nabla\Phi(g)=D^2\Phi(f)(f-g).
\]
Now we differentiate this identity with respect to $g$. 
For all $k\in \mathcal{F}$, the left-hand side gives $-D^2\Phi(g)k$, while the right-hand side gives $-D^2\Phi(f)k$, so $D^2\Phi(f)=D^2\Phi(g)$ for all $f,g\in \mathcal{A}$. 
Thus $D^2\Phi(\cdot)\equiv T$ is constant on $\mathcal{A}$. 
Since each Hessian is self-adjoint and positive, $T$ is self-adjoint and positive.  strict convexity implies $T$ is strictly positive.

Fix $f_0\in A$. From $\nabla\Phi(f)-\nabla\Phi(f_0)=T(f-f_0)$ we obtain
$\nabla\Phi(f)=Tf+b$ with $b=\nabla\Phi(f_0)-Tf_0$. Therefore
$\nabla\!\bigl(\Phi(f)-\frac12\langle f,Tf\rangle_F-\langle b,f\rangle_F\bigr)=0$ on $\mathcal{A}$, so, since $\mathcal{A}$ is convex and hence connected, this quantity is constant on $\mathcal{A}$. Thus
\[
\Phi(f)=\frac12\langle f,Tf\rangle_{\mathcal{F}}+\langle b,f\rangle_{\mathcal{F}}+c
\]
for some $c\in\mathbb{R}$, proving $(2)$. \qedhere
\end{proof}

\kFBD*
\begin{proof}
First, there exists a unique Fr\'echet gradient $\nabla F(z) \in \mathcal{H}$ such that the Fr\'echet derivative $DF(z):\mathcal{H} \to \mathbb{R}$ satisfies $DF(z)[h] = \langle \nabla F(z), h \rangle_{\mathcal{H}}$. 
Second, the chain rule implies that similarly, the Fr\'echet derivatives and gradients of $\Phi$ satisfy
$$D \Phi(f)[h] = DF(\mu(f))[\mu(h)] = \langle \nabla F(\mu(f)), \mu(h) \rangle_{\mathcal{H}}.$$
Using the adjoint $\mu^\ast:\mathcal{H}\to\mathcal{F}$ (see Appendix~\ref{app:frechet}) this then implies that
$$ \nabla \Phi(f) = \mu^\ast \Big(\nabla F\big(\mu(f)\big)\Big) \in \mathcal{F}.$$
Hence using the Riesz representer and the adjoint, we may write the FBD as
    \begin{align*}
        d_\Phi(f,g) &= \Phi(f) - \Phi(g) - \langle \nabla \Phi(g),  f-g \rangle_{\mathcal{F}}\\
        &= F(\mu(f)) - F(\mu(g)) - \langle \nabla F(\mu(g)), \mu(f) - \mu(g) \rangle_{\mathcal H} \\
        &= d_F(\mu(f), \mu(g)).
    \end{align*}
\end{proof}

\sandwich*
\begin{proof}
Fix $p,q\in \mathcal P_R$, and define the line segment $p_t=(1-t)q+tp$ for $t\in[0,1]$. Since $\mathcal{F}$ is convex, we have $p_t\in \mathcal{F}$ for all $t$. Moreover, because the kernel mean embedding is affine,
$\mu(p_t)=(1-t)\mu(q)+t\mu(p)$, and hence
$\|\mu(p_t)\|_{\mathcal H}\le (1-t)\|\mu(q)\|_{\mathcal H}+t\|\mu(p)\|_{\mathcal H}\le R$.
Thus $p_t\in \mathcal P_R$ for all $t\in[0,1]$.

Now set $\gamma(t)=\Phi(p_t)$. Since $\phi\in C^2([0,\infty))$ with $\phi'(0)=0$, the map $u\mapsto \phi(\|u\|_{\mathcal H})$ is $C^2$ on $\mathcal H$, so $\gamma\in C^2([0,1])$. By the integral remainder formula for FBDs,
\[
D_{\Phi}(p,q)
=
\int_0^1 (1-t)\,\gamma''(t)\,dt
=
\int_0^1 (1-t)\,D^2\Phi(p_t)[p-q,p-q]\,dt.
\]
At each $p_t$, the Hessian of $\Phi$ acts on the direction $p-q$ through the embedded difference $\mu(p)-\mu(q)$, and its eigenvalues lie between $m_\phi(R)$ and $L_\phi(R)$ because $\|\mu(p_t)\|_{\mathcal H}\le R$. Therefore
\[
m_\phi(R)\,\|\mu(p)-\mu(q)\|_{\mathcal H}^2
\;\le\;
D^2\Phi(p_t)[p-q,p-q]
\;\le\;
L_\phi(R)\,\|\mu(p)-\mu(q)\|_{\mathcal H}^2
\]
for all $t\in[0,1]$. Integrating against $(1-t)\,dt$ and using $\int_0^1(1-t)\,dt=1/2$ yields
\[
\frac{m_\phi(R)}{2}\,\|\mu(p)-\mu(q)\|_{\mathcal H}^2
\;\le\;
D_{\Phi}(p,q)
\;\le\;
\frac{L_\phi(R)}{2}\,\|\mu(p)-\mu(q)\|_{\mathcal H}^2.
\]
Since $\mathrm{MMD}_k^2(p,q)=\|\mu(p)-\mu(q)\|_{\mathcal H}^2$, the result follows.
\end{proof}

\universalestimationgeneral*
\begin{proof}
By definition of $\widehat\theta_n$,
$\sqrt{D_\Phi(p_{\widehat\theta_n},\widehat p_n)}\le \sqrt{D_\Phi(p_\theta,\widehat p_n)}$ for every $\theta\in\Theta$.
Using the comparison between $D_\Phi$ and $D$, this implies
\[
\sqrt{\frac{m}{2}}\,\sqrt{D(p_{\widehat\theta_n},\widehat p_n)}
\le
\sqrt{D_\Phi(p_{\widehat\theta_n},\widehat p_n)}
\le
\sqrt{D_\Phi(p_\theta,\widehat p_n)}
\le
\sqrt{\frac{L}{2}}\,\sqrt{D(p_\theta,\widehat p_n)},
\]
and hence
$\sqrt{D(p_{\widehat\theta_n},\widehat p_n)}\le \sqrt{L/m}\,\sqrt{D(p_\theta,\widehat p_n)}$.
Applying the triangle inequality for $D$ twice gives
\[
\sqrt{D(p_{\widehat\theta_n},p_0)}
\le
\sqrt{D(p_{\widehat\theta_n},\widehat p_n)}+\sqrt{D(\widehat p_n,p_0)}
\le
\sqrt{\frac{L}{m}}\,\sqrt{D(p_\theta,p_0)}
+
\left(1+\sqrt{\frac{L}{m}}\right)\sqrt{D(\widehat p_n,p_0)}.
\]
Finally, using the upper bound $\sqrt{D_\Phi(p_{\widehat\theta_n},p_0)}\le \sqrt{L/2}\,\sqrt{D(p_{\widehat\theta_n},p_0)}$ yields
\[
\sqrt{D_\Phi(p_{\widehat\theta_n},p_0)}
\le
\frac{L}{\sqrt{2m}}\,\sqrt{D(p_\theta,p_0)}
+
\left(
\sqrt{\frac{L}{2}}+\frac{L}{\sqrt{2m}}
\right)\sqrt{D(\widehat p_n,p_0)}.
\]
Taking the infimum over $\theta$ and then expectations gives the second claim.
\end{proof}

\universalestimation*
\allowdisplaybreaks
\begin{proof}
The inequalities follow~\cite{cherief2022finite}. 
Under Assumption~\ref{ass:kernelised_cov}, we have
\begin{align*}
\mathbb{E}\!\left[\sqrt{D(\widehat p_n,p_0)}\right] &\leq \sqrt{\mathbb{E}\!\left[D(\widehat p_n,p_0)\right]} \\
&= \sqrt{\mathbb{E} \Big\Vert \frac{1}{n} \sum_{i=1}^n k(x_i, \cdot) - \mu(p_0)  \Big\Vert_{\mathcal{H}}^2}  \\
&= \sqrt{\frac{1}{n^2}  \Bigg(\mathbb{E}   \sum_{i=1}^n \Big\Vert k(x_i, \cdot) - \mu(p_0)  \Big\Vert_{\mathcal{H}}^2 + 2 \sum_{1 \leq i < j \leq n} \big\langle k(x_i,\cdot) - \mu(p_0),  k(x_j,\cdot) - \mu(p_0) \big\rangle_{\mathcal{H}} \Bigg)} \\
&\leq \sqrt{\frac{1}{n^2}  \Bigg(\mathbb{E}   \sum_{i=1}^n \Big\Vert k(x_i, \cdot)  \Big\Vert_{\mathcal{H}}^2 + 2 \sum_{1 \leq i < j \leq n} \big\langle k(x_i,\cdot) - \mu(p_0),  k(x_j,\cdot) - \mu(p_0) \big\rangle_{\mathcal{H}} \Bigg)} \\
&= \sqrt{\frac{1}{n^2}  \Bigg(\mathbb{E}   \sum_{i=1}^n \Big\Vert k(x_i, \cdot)  \Big\Vert_{\mathcal{H}}^2 + 2 \sum_{1 \leq i < j \leq n} \rho_{j-i} \Bigg)} \\
&\leq \sqrt{\frac{1}{n^2}  \Bigg(n + 2 \sum_{1 \leq i < j \leq n} \rho_{j-i} \Bigg)}\\
&= \sqrt{\frac{1}{n^2}  \Bigg(n + 2  \sum_{t=1}^{n-1} (n-t) \rho_t \Bigg)} \\
&= \frac{1}{\sqrt n} \sqrt{ 1 + 2\sum_{t=1}^{n-1}(1 - t/n) \rho_t } \\
&\leq \frac{1}{\sqrt n} \sqrt{1 + \rho}
\end{align*}
where the fourth inequality uses $\text{Var}(Z) \leq \mathbb{E}(Z^2)$. 
Proposition~\ref{prop:universal-metric} then gives 
$$ \mathbb E\!\left[\sqrt{D_\Phi(p_{\widehat\theta_n},p_0)}\right]
\le
\inf_{\theta\in\Theta}
\frac{L}{\sqrt{2m}}\,\sqrt{D(p_\theta,p_0)}
+
\sqrt\frac{1+\rho}{n} \left(
\sqrt{\frac{L}{2}}+\frac{L}{\sqrt{2m}}
\right).$$
\end{proof}

\end{document}